\documentclass{article}

\usepackage{microtype}
\usepackage{graphicx}
\usepackage{caption}

\usepackage{booktabs} 

\usepackage{hyperref}

\usepackage{hyperref}
\usepackage{url}
\usepackage{tabularx}

\usepackage{amsmath, amsfonts, amsthm, amssymb}
\usepackage{enumitem}
\usepackage{wrapfig}
\usepackage{graphicx}
\usepackage{subcaption}
\usepackage{comment}
\usepackage{nicefrac}

\newtheorem{theorem}{Theorem}

\newtheorem{lemma}{Lemma}

\newtheorem{remark}{Remark}
\newtheorem{definition}{Definition}
\newtheorem{proposition}{Proposition}

\theoremstyle{definition}
\newtheorem{experiment}{Experiment}

\usepackage[colorinlistoftodos, textsize=tiny]{todonotes}

\newcommand{\beq}{\begin{equation}}
\newcommand{\eeq}{\end{equation}}
\newcommand{\be}{\begin{equation}}
\newcommand{\ee}{\end{equation}}
\newcommand{\beqa}{\begin{eqnarray}}
\newcommand{\eeqa}{\end{eqnarray}}
\newcommand{\bean}{\begin{eqnarray*}}
\newcommand{\eean}{\end{eqnarray*}}

\newcommand{\cM}{{\mathcal M}}

\newcommand{\R}{\mathbb{R}}

\def\to{\rightarrow}

\newcommand{\x}{\mathbf{x}}
\newcommand{\z}{\mathbf{z}}

\newcommand{\uu}{\mathbf{u}}
\newcommand{\kk}{\mathbf{k}}
\newcommand{\bl}{\mathbf{l}}
\newcommand{\y}{\mathbf{y}}

\def \1{\mathds{1}}
\newcommand{\sub}[2]{#1^{(#2)}}

\usepackage[accepted]{icml2019}

\icmltitlerunning{On the Spectral Bias of Neural Networks}

\begin{document}

\twocolumn[
\icmltitle{On the Spectral Bias of Neural Networks}



\icmlsetsymbol{equal}{*}

\begin{icmlauthorlist}
\icmlauthor{Nasim Rahaman}{equal,mila,ial}
\icmlauthor{Aristide Baratin}{equal,mila}
\icmlauthor{Devansh Arpit}{mila}
\icmlauthor{Felix Draxler}{ial}
\icmlauthor{Min Lin}{mila}
\icmlauthor{Fred A.~Hamprecht}{ial}
\icmlauthor{Yoshua Bengio}{mila}
\icmlauthor{Aaron Courville}{mila}
\end{icmlauthorlist}

\icmlaffiliation{mila}{Mila, Quebec, Canada}
\icmlaffiliation{ial}{Image Analysis and Learning Lab, Ruprecht-Karls-Universit\"at Heidelberg, Germany}

\icmlcorrespondingauthor{Nasim Rahaman}{nasim.rahaman@live.com}
\icmlcorrespondingauthor{Aristide Baratin}{aristide.baratin@umontreal.ca}
\icmlcorrespondingauthor{Devansh Arpit}{devansharpit@gmail.com}

\icmlkeywords{Machine Learning, ICML}

\vskip 0.3in
]



\printAffiliationsAndNotice{\icmlEqualContribution} 

\newcommand{\fix}{\marginpar{FIX}}
\newcommand{\new}{\marginpar{NEW}}
\begin{abstract}
Neural networks are known to be a class of highly expressive functions  able to fit even random input-output mappings with $100\%$ accuracy. 
In this work we present properties of neural networks that complement this aspect of expressivity. By using tools from Fourier analysis, we highlight a learning bias of deep networks towards low frequency functions -- i.e. functions that vary globally without local fluctuations -- which  manifests itself as a frequency-dependent learning speed. Intuitively, this property is in line with the observation that over-parameterized networks prioritize learning simple patterns that generalize across data samples. We also investigate the role of the shape of the data manifold by presenting empirical and theoretical evidence that, somewhat counter-intuitively, learning higher frequencies gets \emph{easier} with increasing manifold complexity.

\end{abstract}
\section{Introduction}

The remarkable success of deep neural networks at generalizing to natural data 
is at odds with the traditional notions of model complexity and their empirically demonstrated ability to fit arbitrary random data to perfect accuracy \citep{understanding_DL, arpit2017closer}. This has prompted recent investigations of possible implicit regularization mechanisms inherent in the learning process which induce a bias towards low complexity solutions \citep{neyshabur2014search, soudry2017implicit, poggio2018theory, neyshabur2017exploring}.

In this work, we take a slightly shifted view on implicit regularization by suggesting that 
low-complexity functions are \emph{learned faster} during training by gradient descent. 
We expose this bias by taking a closer look at neural networks through the lens of Fourier analysis. 
While they can approximate arbitrary functions, we find that these networks prioritize learning the low frequency modes, 
a phenomenon we call the \emph{spectral bias}. 
This bias manifests itself not just in the process of learning, but also in the parameterization of the model itself: 
in fact, we show that the lower frequency components of trained networks are more robust to random parameter perturbations. Finally, we also expose and analyze the rather intricate interplay between the spectral bias and the geometry of the data manifold by showing that high frequencies get easier to learn when the data lies on a lower-dimensional manifold of complex shape embedded in the input space of the model. 
We focus the discussion on networks with rectified linear unit (ReLU) activations, whose continuous piece-wise linear structure enables an analytic treatment. 

\subsection*{Contributions\footnote{Code: \hyperlink{https://github.com/nasimrahaman/SpectralBias}{https://github.com/nasimrahaman/SpectralBias}}}

\begin{enumerate}

\item We exploit the continuous piecewise-linear structure of ReLU networks to evaluate 
its Fourier spectrum (Section \ref{sec:asymptotics}). 
\item We find empirical evidence of a \emph{spectral bias}: i.e. lower frequencies are learned first. We also show that lower frequencies are more robust to random perturbations of the network parameters (Section \ref{sec:lowfreqfirst}).  
\item We study the role of the shape of the data manifold: we  show how complex manifold shapes can facilitate the learning of higher frequencies and  develop a theoretical understanding of this behavior
(Section \ref{sec:notallmfds}).

\end{enumerate}

\section{Fourier analysis of ReLU networks} \label{sec:asymptotics}

\subsection{Preliminaries}

Throughout the paper we call `ReLU network' a scalar function $f: \R^d  \mapsto \R$ defined by a neural network with $L$ hidden layers of widths  $d_1, \cdots d_L$ 
and a single output neuron: 
\beq \label{reluDNN2}
f(\x) = \left(T^{(L+1)} \circ \sigma \circ T^{(L)} \circ \cdots \circ \sigma \circ T^{(1)}\right)(\x)
\eeq
where each $\sub{T}{k}: \R^{d_{k-1}} \to \R^{d_k}$ is an affine function 
($d_0 = d$ and $d_{L+1} = 1$) and $\sigma(\uu)_i = \max(0, u_i)$ denotes the ReLU activation function acting elementwise on a vector $\uu = (u_1, \cdots u_n)$. In the standard basis, $\sub{T}{k}(\x) = W^{(k)} \x + \mathbf{b}^{(k)} $ for some weight matrix $W^{(k)}$  and bias vector $\mathbf{b}^{(k)}$. 

ReLU networks are known to be continuous piece-wise linear (CPWL) functions, where the linear regions are convex polytopes  \citep{raghu2016expressive, montufar2014number, zhang2018tropical, arora2018understanding}.  
Remarkably, the converse also holds: 
every CPWL function can be represented by a ReLU network \citep[Theorem 2.1]{arora2018understanding}, which in turn endows ReLU networks with universal approximation properties. 
Given the ReLU network $f$ from Eqn.~\ref{reluDNN2}, we  
can make the piecewise linearity explicit by writing,
\beq \label{eqn:cpwlrelu}
f(\x) = \sum_{\epsilon} 1_{P_{\epsilon}}(\x) \, (W_{\epsilon} \x + \mathbf{b}_\epsilon) 
\eeq
where $\epsilon$ is an index for the linear regions $P_\epsilon$ 
and $1_{P_\epsilon}$ is the indicator function on $P_\epsilon$. As shown in Appendix \ref{app:cpwlrelu} in more detail, each region corresponds to an \emph{activation pattern}\footnote{We adopt the terminology of \citet{raghu2016expressive, montufar2014number}.} of all hidden neurons of the network, which is a binary vector with components conditioned on the sign of the input of the respective neuron.  The $1 \times d$ matrix $W_{\epsilon}$ is given by
\beq 
W_\epsilon = W^{(L+1)} W^{(L)}_{\epsilon} \cdots W^{(1)}_{\epsilon} 
\eeq
where $W^{(k)}_{\epsilon}$ is obtained from the original weight $W^{(k)}$ by setting its $j^{th}$ column to zero whenever the neuron $j$ of the $k^{th}$ layer is inactive. 

\subsection{Fourier Spectrum} 

In the following, we study the structure of ReLU networks in terms of their Fourier representation,  $f(\x) := (2\pi)^{\nicefrac{d}{2}} \int \tilde{f}(\kk) \, e^{i\kk\cdot \x} \mathbf{dk}$, where $\tilde{f}(\kk) := \int f(\x)\, e^{- i \kk\cdot \x} \mathbf{dx}$ is the Fourier transform\footnote{Note that general ReLU networks need not be squared integrable: for instance, the class of two-layer ReLU networks represent an arrangement of hyperplanes \citep{montufar2014number} and hence grow linearly as $x \to \infty$. In such cases, the Fourier transform is to be understood in the sense of tempered distributions acting on rapidly decaying smooth functions $\phi$ as $\langle \tilde{f}, \phi \rangle = \langle f, \tilde{\phi} \rangle$. See Appendix \ref{app:ftrelu} for a formal treatment.
}. Lemmas \ref{FTRelu} and \ref{lemma:ftpolyrat} 
yield the explicit form of the Fourier components (we refer to Appendix \ref{app:ftrelu} for the proofs and technical details). 

\begin{lemma} \label{FTRelu} The Fourier transform of ReLU networks decomposes as,
\beq \label{FTReluform} 
\tilde{f}(\mathbf{k}) = i \sum_{\epsilon} \frac{W_\epsilon \mathbf{k}}{k^2} 
\,\tilde{1}_{P_{\epsilon}}(\mathbf{k})
\eeq
where $k=\|\mathbf{k}\|$ and $\tilde{1}_{P}(\mathbf{k}) = \int_{P} e^{-i\mathbf{k}\cdot \mathbf{x}} \mathbf{dx}$ is the Fourier transform of the indicator function of $P$. 
\end{lemma}

The Fourier transform of the indicator over linear regions appearing in Eqn.~\ref{FTReluform} are fairly intricate mathematical objects. \citet{diaz2016fourier} develop an elegant procedure for evaluating it in arbitrary dimensions via a recursive application of Stokes theorem. We describe this procedure in detail\footnote{We also  generalize the construction to tempered distributions.} in Appendix \ref{app:ftpolytope}, and present here its main corollary.

\begin{lemma} \label{lemma:ftpolyrat}
Let $P$ be a full dimensional polytope in $\mathbb{R}^d$. Its Fourier spectrum takes the form: 
\beq 
\tilde{1}_P(\mathbf{k}) = \sum_{n=0}^{d} \frac{D_n (\kk) 1_{G_n}(\kk)}{k^{n}}
\eeq
where $G_n$ 
is the union of $n$-dimensional subspaces that are orthogonal to some $n$-codimensional face of $P$,  
$D_n: \R^d \to \mathbb{C}$ is in  $\Theta(1)\,(k \to \infty)$  and $1_{G_n}$ the indicator over $G_n$. 
\end{lemma}


Lemmas \ref{FTRelu}, \ref{lemma:ftpolyrat} 
together yield the main result of this section. 

\begin{theorem} 
\label{theorem:spectraldecay}
The Fourier components of the ReLU network $f_{\theta}$ 
with parameters $\theta$ is given by the rational function:
\beq \label{FTbound}
\tilde{f}_{\theta}(\kk) = \sum_{n=0}^{d} \frac{C_n(\theta, \kk) 1_{H^\theta_n}(\kk)}{k^{n + 1}}
\eeq
where 
$H^\theta_n$ is the union of $n$-dimensional subspaces that are orthogonal to some $n$-codimensional faces of some polytope $P_{\epsilon}$ and $C_n(\cdot, \theta): \R^{d} \to \mathbb{C}$ is
$\Theta(1)\,(k \to \infty)$. 
\end{theorem}

Note that Eqn~\ref{FTbound} applies to general ReLU networks with arbitrary width and depth\footnote{Symmetries that might arise due to additional assumptions can be used to further develop Eqn~\ref{FTbound}, see e.g. \citet{eldan2016power} for 2-layer networks.}.

\textbf{Discussion.} We make the following two observations. First, the spectral decay of ReLU networks is highly anisotropic in large dimensions. In almost all directions of $\R^d$, we have a $k^{-d-1}$ decay. However, the decay can be as slow as $k^{-2}$ in specific directions 
orthogonal to the $d-1$ dimensional faces bounding the linear regions\footnote{Note that such a rate is \emph{not} guaranteed by piecewise smoothness alone. For instance, the function $\sqrt{|x|}$ is continuous and smooth everywhere except at $x = 0$, yet it decays as $k^{-1.5}$ in the Fourier domain.}. 

Second,  the numerator in Eqn~\ref{FTbound} is bounded by $N_f L_f$ (cf. Appendix~\ref{app:moarspectraldecay}), where $N_f$ is the number of linear regions and $L_f = \max_{\epsilon} \|W_{\epsilon}\|$ is the Lipschitz constant of the network.
Further, the Lipschitz constant $L_f$ can be bounded as (cf. Appendix~\ref{App:Lipschitzbound}): 
\beq
L_f \le \prod_{k = 1}^{L+1} \|W^{(k)}\| \le \|\theta\|_{\infty}^{L+1} \sqrt{d} \prod_{k=1}^{L} d_k 
\eeq
where $\|\cdot\|$ is the spectral norm and $\|\cdot\|_{\infty}$ the max norm, and $d_k$ is the number of units in the $k$-th layer. This makes the bound on $L_f$ scale exponentially in depth and polynomial in width. As for the number $N_f$ of linear regions, \citet{montufar2014number} and \citet{raghu2016expressive} obtain tight bounds that exhibit the same scaling behaviour \citep[Theorem 1]{raghu2016expressive}. In Appendix~\ref{app:arch_abl}, we qualitatively ablate over the depth and width of the network to expose how this reflects on the Fourier spectrum of the network. 

\begin{figure*}[t]
\begin{subfigure}[t]{0.48\textwidth}
\centering
\includegraphics[width=0.48\linewidth,trim=0.in 0.1in 0.1in 0.1in,clip]{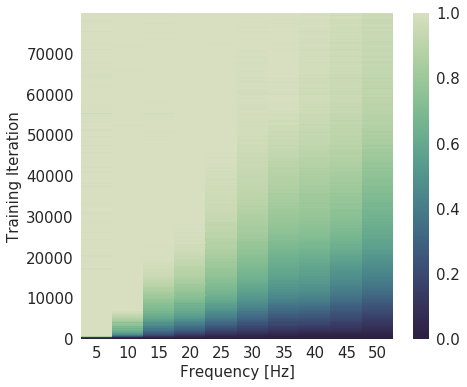} \,
\includegraphics[width=0.48\linewidth,trim=0.in 0.1in 0.1in 0.1in,clip]{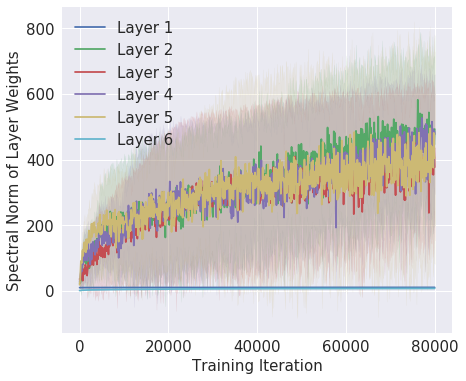}
\caption{\small Equal Amplitudes}
\end{subfigure}
\hfill
\begin{subfigure}[t]{0.48\textwidth}
\centering
\includegraphics[width=0.48\linewidth,trim=0.in 0.1in 0.1in 0.1in,clip]{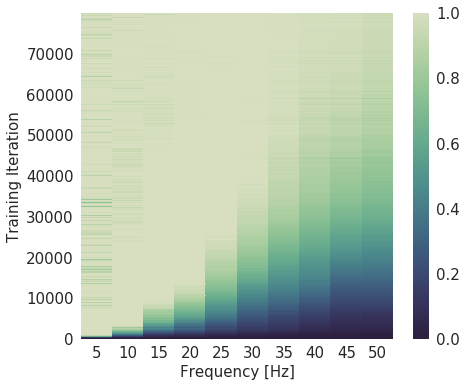} \,
\includegraphics[width=0.48\linewidth,trim=0.in 0.1in 0.1in 0.1in,clip]{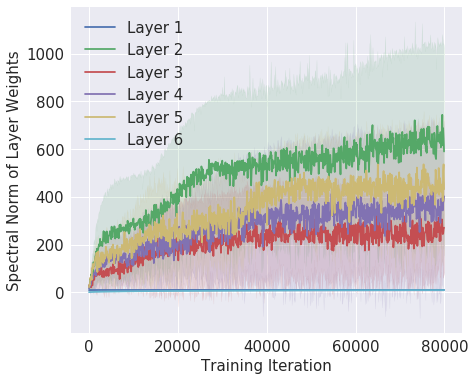}
\caption{\small Increasing Amplitudes}
\end{subfigure}
\caption{\small Left (a, b): Evolution of the spectrum (x-axis for frequency) during training (y-axis). The colors show the measured amplitude of the network spectrum at the corresponding frequency, normalized by the target amplitude at the same frequency (i.e. $|\tilde{f}_{k_i}|/A_i$) and the colorbar is clipped between 0 and 1. Right (a, b): Evolution of the spectral norm (y-axis) of each layer during training (x-axis). Figure-set (a) shows the setting where all frequency components in the target function have the same amplitude, and (b) where higher frequencies have larger amplitudes. \textbf{Gist}: We find that even when higher frequencies have larger amplitudes, the model prioritizes learning lower frequencies first. We also find that the spectral norm of weights increases as the model fits higher frequency, which is what we expect from Theorem \ref{theorem:spectraldecay}.} \label{fig:lowfreqfirst}
\end{figure*}

\section{Lower Frequencies are Learned First} \label{sec:lowfreqfirst}
We now present experiments 
showing  that networks tend to fit \emph{lower frequencies first} during training. We refer to this phenomenon as the \emph{spectral bias}, and discuss it in light of the results of
Section \ref{sec:asymptotics}. 

\begin{figure*}
\begin{subfigure}[t]{0.24\textwidth}
\centering
\includegraphics[width=0.99\textwidth]{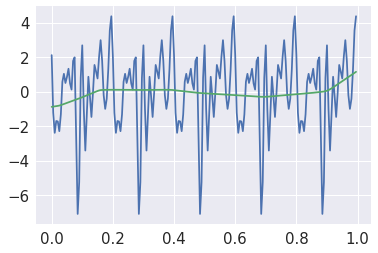}
\caption{\small Iteration 100}
\end{subfigure}
\hfill
\begin{subfigure}[t]{0.24\textwidth}
\centering
\includegraphics[width=0.99\textwidth]{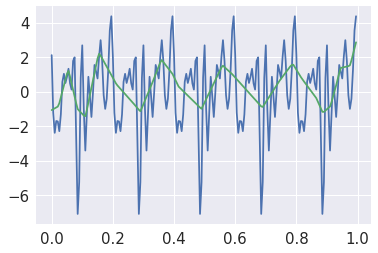}
\caption{\small Iteration 1000}
\end{subfigure}
\hfill
\begin{subfigure}[t]{0.24\textwidth}
\centering
\includegraphics[width=0.99\textwidth]{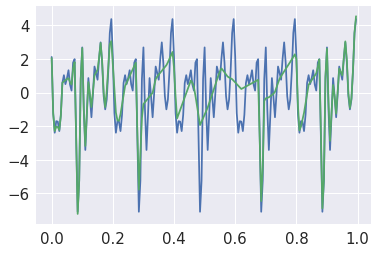}
\caption{\small Iteration 10000}
\end{subfigure}
\hfill
\begin{subfigure}[t]{0.24\textwidth}
\centering
\includegraphics[width=0.99\textwidth]{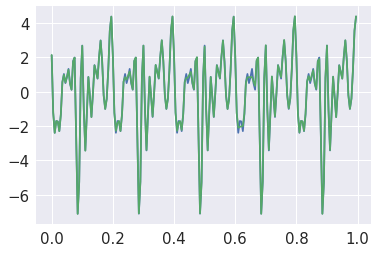}
\caption{\small Iteration 80000}
\end{subfigure}
\caption{\small The learnt function (green) overlayed on the target function (blue) as the training progresses. The target function is a superposition of sinusoids of frequencies $\kappa = (5, 10, ..., 45, 50)$, equal amplitudes and randomly sampled phases.} \label{fig:learntfuncs_eq_amps}
\end{figure*}

\subsection{Synthetic Experiments}

\begin{experiment} \label{experiment:lowfreqfirst} The setup is as follows\footnote{More experimental details and additional plots are provided in Appendix \ref{app:experiment:lowfreqfirst}.}: 
Given frequencies $\kappa = (k_1, k_2, ...)$ with corresponding amplitudes $\alpha = (A_1, A_2, ...)$, and phases $\phi = (\varphi_1, \varphi_2, ...)$, we consider the mapping $\lambda: [0, 1] \to \mathbb{R}$ given by
\beq  \label{target}
\lambda(z) = \sum_i A_i \sin(2\pi k_i z + \varphi_i).
\eeq
A 6-layer deep 256-unit wide ReLU network $f_\theta$ is trained to regress $\lambda$ with $\kappa = (5, 10, ..., 45, 50)$ and $N=200$ input samples spaced equally over $[0, 1]$; 
its spectrum $\tilde{f}_\theta(k)$ in expectation over $\varphi_i \sim U(0, 2 \pi)$ is monitored as training progresses. In the first setting, we set equal amplitude $A_i = 1$ for all frequencies and in the second setting, the amplitude increases from $A_1 = 0.1$ to $A_{10} = 1$. Figure~\ref{fig:lowfreqfirst} shows the normalized magnitudes $|\tilde{f}_\theta(k_i)|/ A_i$ at various frequencies, as training progresses with full-batch gradient descent. Further, Figure~\ref{fig:learntfuncs_eq_amps} shows the learned function at intermediate training iterations. The result is that 
lower frequencies (i.e. smaller $k_i$'s) are regressed first, regardless of their amplitudes. 
\end{experiment} 

\begin{experiment} \label{experiment:robustness}
\begin{figure}[!h]
\centering
\includegraphics[width=0.3\textwidth]{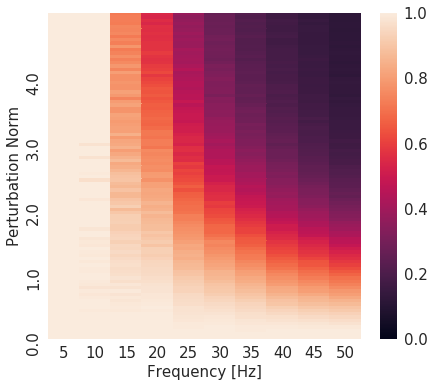}
\caption{\small Normalized spectrum of the model (x-axis for frequency, colorbar for magnitude) with perturbed parameters as a function of parameter perturbation (y-axis). The colormap is clipped between 0 and 1. We observe that the lower frequencies are more robust to parameter perturbations than the higher frequencies. \label{fig:robustness}}
\end{figure}
Our goal here is to illustrate a phenomenon that complements the one highlighted above: lower frequencies are more \emph{robust} to parameter perturbations. The set up is the same as in Experiment \ref{experiment:lowfreqfirst}. The network is trained to regress a target function with frequencies $\kappa = (10, 15, 20, ..., 45, 50)$ and amplitudes $A_i = 1 \, \forall \, i$. 
After convergence to $\theta^{*}$, we consider random (isotropic) perturbations $\theta = \theta^{*} + \delta \hat{\theta}$ of given magnitude $\delta$, 
where  $\hat{\theta}$ is a random unit vector in parameter space. We evaluate the network function $f_{\theta}$  at the perturbed parameters, and compute the magnitude of its discrete Fourier transform at frequencies $k_i$ to obtain $|\tilde{f}_{\theta}({k_i})|$. We also average over 100 samples of $\hat{\theta}$ to obtain $|\tilde{f}_{\mathbb{E}\theta}({k_i})|$, which we normalize by $|\tilde{f}_{\theta*}({k_i})|$. Finally, we average over the phases $\phi$ (see Eqn~\ref{target}). The result, shown in Figure \ref{fig:robustness}, demonstrates that higher frequencies are significantly less robust than the lower ones, guiding the intuition that expressing higher frequencies requires the parameters to be finely-tuned to work together. In other words, parameters that contribute towards expressing high-frequency components occupy a small volume in the parameter space. We formalize this  in Appendix \ref{app:volparamspace}. 
\end{experiment}

\begin{figure*}[t]
\centering
\begin{subfigure}[t]{0.245\textwidth}
\includegraphics[width=1\linewidth]{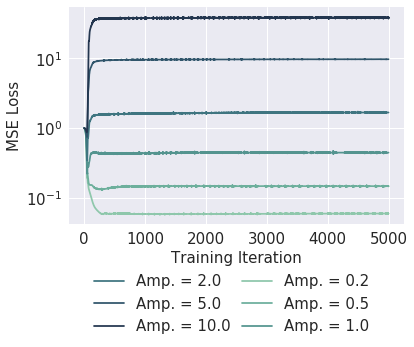}
\caption{\small $k = 0.1$ \label{fig:mnistnoise_a}}
\end{subfigure}
\begin{subfigure}[t]{0.245\textwidth}
\includegraphics[width=1\linewidth]{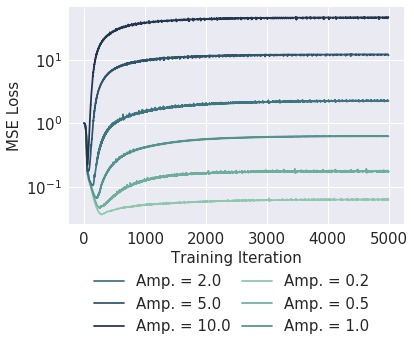}
\caption{\small $k = 1$ \label{fig:mnistnoise_b}}
\end{subfigure}
\begin{subfigure}[t]{0.245\textwidth}
\centering
\includegraphics[width=1\textwidth]{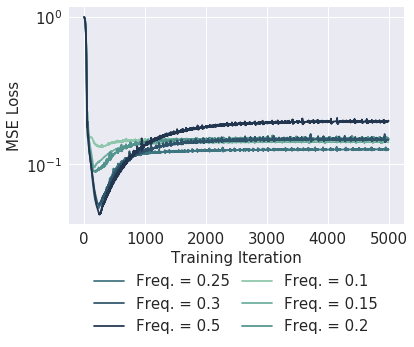}
\caption{\small $\beta = 0.5$ \label{fig:mnistnoise_c}}
\end{subfigure}
\begin{subfigure}[t]{0.245\textwidth}
\centering
\includegraphics[width=1\textwidth]{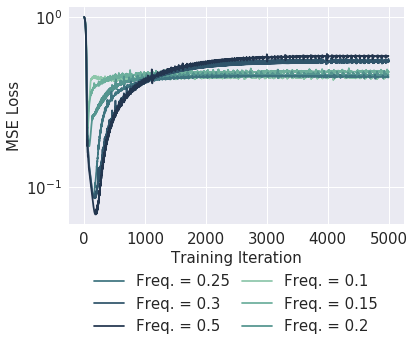}
\caption{\small $\beta = 1.$ \label{fig:mnistnoise_d}}
\end{subfigure}
\caption{\small (a,b,c,d): Validation curves for various settings of noise amplitude $\beta$ and frequency $k$. Corresponding training curves can be found in Figure~\ref{fig:mnistnoise_train} in appendix~\ref{app:lohifreqnoisemnist}.
\textbf{Gist}: Low frequency noise affects the network more than their high-frequency counterparts. Further, for high-frequency noise, one finds that the validation loss dips early in the training. Both these observations are explained by the fact that network readily fit lower frequencies, but learn higher frequencies later in the training.} \label{fig:mnistnoise}
\end{figure*}

{\bf Discussion .} Multiple theoretical aspects may underlie these observations. First, for a fixed architecture, recall that the numerator in Theorem \ref{theorem:spectraldecay} is\footnote{The tightness of this bound is verified empirically in appendix~\ref{app:arch_abl}.} $\mathcal{O}(L_f)$ (where $L_f$ is the Lipschitz constant of the function). However, $L_f$ is bounded by the parameter norm, which can only increase gradually during training by gradient descent. This leads to the higher frequencies being learned\footnote{This assumes that the Lipschitz constant of the (noisy) target function is larger than that of the network at initialization.} late in the optimization process. To confirm that the bound indeed increases as the model fits higher frequencies, we plot in Fig~\ref{fig:lowfreqfirst} the spectral norm of weights of each layer during training for both cases of constant and increasing amplitudes. 

Second (cf. Appendix~\ref{app:graddecay}), the exact form of the Fourier spectrum yields that for a fixed direction $\hat \kk$, the spectral decay rate of the parameter gradient $\nicefrac{\partial \tilde f}{\partial \theta}$ is at most one exponent  of $k$ lower than that of $\tilde{f}$. If for a fixed $\hat \kk$ we have $\tilde f = \mathcal{O}(k^{-\Delta - 1})$ where $1 \le \Delta \le d$, we obtain for the residual $h = f - \lambda$ and (continuous) training step $t$: 
\begin{align} \label{eq:residualdecay}
\left| \frac{d \tilde h(\kk)}{d t} \right| = \left| \frac{d\tilde f(\kk)}{d t} \right| = \underbrace{\left| \frac{d \tilde f(\kk)}{d \theta} \right|}_{ \mathcal{O}(k^{-\Delta})} \overbrace{\left| \frac{d \theta}{d t} \right|}^{\left| \eta \cdot \nicefrac{d \mathcal L}{d \theta}\right|} = \mathcal{O}(k^{-\Delta}) 
\end{align}
where we use the fact that $\nicefrac{d \theta}{d t}$ is just the learning rate times the parameter gradient of the loss which is independent\footnote{Note however that the loss term might involve a sum or an integral over all frequencies, but the summation is over a different variable.} of $k$, and assume that the target function $\lambda$ is fixed. Eqn~\ref{eq:residualdecay} shows that the rate of change of the residual decays with increasing frequency, which is what we find in Experiment~\ref{experiment:lowfreqfirst}.

\subsection{Real-Data Experiments}

While Experiments~\ref{experiment:lowfreqfirst} and \ref{experiment:robustness} establish the spectral bias by explicitly evaluating the Fourier coefficients, doing so becomes prohibitively expensive for larger $d$ (e.g. on MNIST). To tackle this, we propose the following set of experiments to measure the effect of spectral bias indirectly on MNIST.

\begin{experiment} \label{experiment:lohifreqnoisemnist}
In this experiment, we investigate how the validation performance dependent on the frequency of noise added to the training target. We find that the best validation performance on MNIST is particularly insensitive to the magnitude of high-frequency noise, yet it is adversely affected by low-frequency noise. We consider a target (binary) function $\tau_{0}: X \to \{0, 1\}$ defined on the space $X = [0, 1]^{784}$ of MNIST inputs.  
Samples $\{\x_i, \tau_0(\x_i)\}_{i}$ form a subset of the MNIST dataset comprising samples $\x_i$ belonging to two classes. Let $\psi_k(\x)$ be a \emph{noise function}:
\beq
\psi_k(\x) = \sin(k\|\x\|)
\eeq
corresponding to a \emph{radial wave} defined on the $784$-dimensional input space\footnote{The rationale behind using a radial wave is that it induces oscillations (simultaneously) along all spatial directions. Another viable option is to induce oscillations along the principle axes of the data: we have verified that the key trends of interest are preserved.}. The final target function $\tau_k$ is then given by $\tau_k = \tau_0 + \beta \psi_k$, where $\beta$ is the effective amplitude of the noise. We fit the same network as in Experiment \ref{experiment:lowfreqfirst} to the target $\tau_k$ with the MSE loss. In the first set of experiments, we ablate over $k$ for a pair of fixed $\beta$s, while in the second set we ablate over $\beta$ for a pair of fixed $k$s. In Figure~\ref{fig:mnistnoise}, we show the respective validation loss curves, where the validation set is obtained by evaluating $\tau_0$ on a separate subset of the data, i.e. $\{\x_j, \tau_0(\x_j)\}_j$. Figure~\ref{fig:mnistnoise_train} (in appendix \ref{app:lohifreqnoisemnist}) shows the respective training curves. 
\end{experiment}

\textbf{Discussion}. 
The profile of the loss curves varies significantly with the frequency of noise added to the target. In Figure~\ref{fig:mnistnoise_a}, we see that the validation performance is adversely affected by the amplitude of the low-frequency noise, whereas Figure~\ref{fig:mnistnoise_b} shows that the amplitude of high-frequency noise does not significantly affect the best validation score. This is explained by the fact that the network readily fits the noise signal if it is low frequency, whereas the higher frequency noise is only fit later in the training. In the latter case, the dip in validation score early in the training is when the network has learned the low frequency true target function $\tau_0$; the remainder of the training is spent learning the higher-frequencies in the training target $\tau$, as we shall see in the next experiment. Figures~\ref{fig:mnistnoise_c} and \ref{fig:mnistnoise_d} confirm that the dip in validation score exacerbates for increasing frequency of the noise. Further, we observe that for higher frequencies (e.g. $k = 0.5$), increasing the amplitude $\beta$ does not significantly degrade the best performance at the dip, confirming that the network is fairly robust to the amplitude of high-frequency noise.

Finally, we note that the dip in validation score was also observed by \citet{arpit2017closer} with i.i.d. noise\footnote{Recall that i.i.d. noise is white-noise, which has a constant Fourier spectrum magnitude in expectation, i.e. it also contains high-frequency components.} in a classification setting. 

\begin{experiment} \label{experiment:kernel}
To investigate the dip observed in Experiment~\ref{experiment:lohifreqnoisemnist}, we now take a more direct approach by considering a generalized notion of frequency. To that end, we project the network function to the space spanned by the orthonormal eigenfunctions $\varphi_n$ of the Gaussian RBF kernel \cite{braun2006model}. These eigenfunctions $\varphi_n$ (sorted by decreasing eigenvalues) resemble sinusoids \cite{fasshauer2011positive}, and the index $n$ can be thought of as being a proxy for the frequency, as can be seen from Figure~\ref{fig:freqofevecs} (see Appendix~\ref{app:kernel} for additional details and supporting plots). While we will call $\tilde{f}[n]$ as the spectrum of the function $f$, it should be understood as $\tilde{f}[n] = \langle f_{\mathcal{H}}, \varphi_n \rangle_{\mathcal{H}}$, where $f_{\mathcal{H}} \in \text{span}\{\varphi_n\}_n$ and $f_{\mathcal{H}}(\x_i) = f(\x_i)$ on the MNIST samples $\x_i \in X$. This allows us to define a noise function as: 
\beq
\psi_{\gamma}(\x) = \sum_{n}^{N} \left(\frac{n}{N}\right)^{\gamma} \varphi_n(\x)
\eeq
where $N$ is the number of available samples and $\gamma = 2$. Like in Experiment~\ref{experiment:lohifreqnoisemnist}, the target function is given by $\tau = \tau_0 + \beta \psi$, and the same network is trained to regress $\tau$. Figure~\ref{fig:kernelhfn} shows the (generalized) spectrum $\tau$ and $\tau_0$, and that of $f$ as training progresses. Figure~\ref{fig:kernelhfn_loss} (in appendix) shows the corresponding dip in validation loss, where the validation set is same as the training set but with true target function $\tau_0$ instead of the noised target $\tau$.
\end{experiment}

\begin{figure}[!htb]
\centering
  \includegraphics[width=1.0\linewidth]{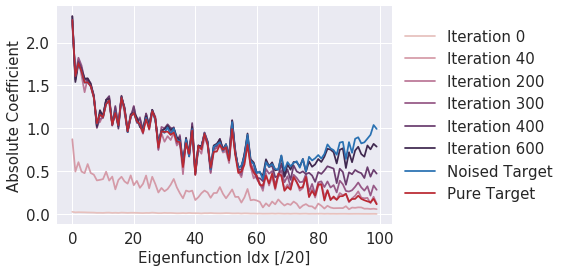}
  \caption{\small Spectrum of the network as it is trained on MNIST target with high-frequency noise (\emph{Noised Target}). We see that the network fits the true target at around the $200$th iteration, which is when the validation score dips (Figure~\ref{fig:kernelhfn_loss} in appendix). \label{fig:kernelhfn}}
\end{figure}
\begin{figure}[!h]
\centering
\includegraphics[width=0.25\textwidth]{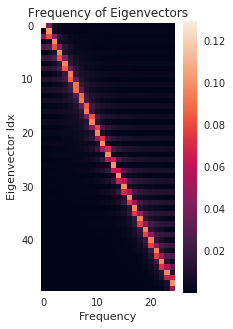}
\caption{\small Spectrum (x-axis for frequency, colorbar for magnitude) of the $n$-th (y-axis) eigenvector of the Gaussian RBF kernel matrix $K_{ij} = k(\x_i, \x_j)$, where the sample set is $\{x_i \in [0, 1]\}_{i = 1}^{50}$ is $N = 50$ uniformly spaced points between $0$ and $1$ and $k$ is the Gaussian RBF kernel function. \textbf{Gist:} The eigenfunctions with increasing $n$ roughly correspond to sinusoids of increasing frequency. Refer to Appendix~\ref{app:kernel} for more details. 
\label{fig:freqofevecs}} 
\end{figure}
\textbf{Discussion}. From Figure~\ref{fig:kernelhfn}, we learn that the drop in validation score observed in Figure~\ref{fig:mnistnoise} is exactly when the higher-frequencies of the noise signal are yet to be learned. As the network gradually learns the higher frequency eigenfunctions, the validation loss increases while the training loss continues to decrease. Thus these experiments show that the phenomenon of spectral bias persists on non-synthetic data and in high dimensional input spaces. 

\section{Not all Manifolds are Learned Equal} \label{sec:notallmfds}

In this section, we investigate subtleties that arise when the data lies on a lower dimensional manifold embedded in the higher dimensional input space of the model. 
We find that the \emph{shape} of the data-manifold impacts  the learnability of high frequencies  in a non-trivial way. As we shall see, this is because low frequency functions in the input space may have high frequency components when restricted to lower dimensional manifolds of complex shapes.
We demonstrate results in an illustrative minimal setting\footnote{We include additional experiments on MNIST and CIFAR-10 in appendices \ref{app:mnist} and \ref{app:cifar10connected}.}, free from unwanted confounding factors, and present a theoretical analysis of the phenomenon.

\vspace{0.1cm}
{\bf Manifold hypothesis.} We consider the case where the data lies on a lower dimensional \emph{data manifold} $\cM \subset \R^d$ embedded in input space \citep{Goodfellow-et-al-2016}, which we assume to be the image  
$\gamma([0,1]^m)$ of some injective mapping $\gamma : [0,1]^m \to \R^d$ defined on a lower dimensional latent space $[0, 1]^m$. Under this hypothesis and in the context of the standard regression problem, a target function $\tau : \cM \to \R$ defined on the data manifold can be identified with a function $\lambda = \tau \circ \gamma$ defined on the latent space. Regressing $\tau$ is therefore equivalent to finding $f : \R^d \to \R$ such that $f \circ \gamma$ matches  $\lambda$. Further, assuming that the data probability distribution $\mu$ supported on $\cM$ is induced by $\gamma$ from the uniform distribution $U$ in the latent space $[0,1]^m$, the mean square error can be expressed as:
\begin{align}
&\textup{MSE}^{(\x)}_{\mu}[f, \tau] = \mathbb{E}_{\x \sim\mu}|f(\x) - \tau(\x)|^2 = \nonumber \\
&\mathbb{E}_{\z \sim U} |(f(\gamma(\z)) - \lambda(\z)|^2 = \textup{MSE}^{(\z)}_{U}[f \circ \gamma, \lambda]
\end{align}
Observe that there is a vast space of degenerate solutions $f$ that minimize the mean squared error -- namely all functions on $\R^d$ that yield the same function when restricted to the data manifold $\cM$.

Our findings from the previous section suggest that neural networks are biased towards expressing a particular subset of such solutions, namely those that are low frequency. It is also worth noting that there exist methods that restrict the space of solutions: notably adversarial training \citep{goodfellow2014explaining} and Mixup \citep{zhang2017mixup}.

\vspace{0.1cm}
{\bf Experimental set up.} The experimental setting is designed to afford control over both the shape of the data manifold and the target function defined on it. We will consider the family of curves in $\R^2$ generated by mappings $\gamma_L: [0,1] \to \R^2$ given by
\begin{align} \label{eqn:toy_mfd_dataset}
\gamma_L(z) = &R_L(z) (\cos(2 \pi z), \sin(2 \pi z)) \nonumber \\ 
\textup{where} \; &R_L(z) = 1 + \frac{1}{2}\sin(2\pi L z)
\end{align}
%
%
\begin{figure}[t]
\centering
\begin{subfigure}{0.23\textwidth}
\includegraphics[width=1.0\linewidth]{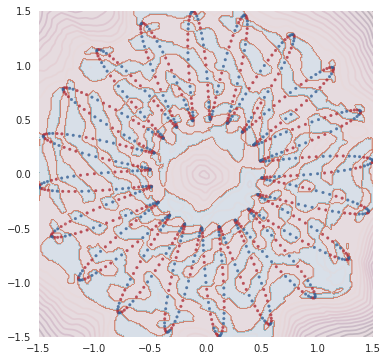}
\end{subfigure}
\begin{subfigure}{0.23\textwidth}
\includegraphics[width=1.0\linewidth]{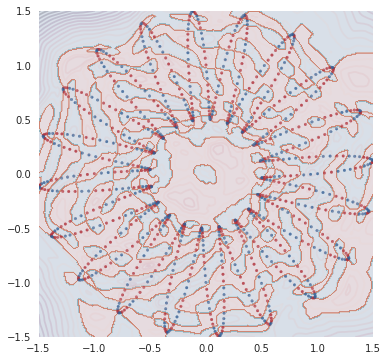}
\end{subfigure}
\caption{\small Functions learned by two identical networks (up to initialization) to classify the binarized value of a sine wave of frequency $k = 200$ defined on a $\gamma_{L = 20}$ manifold. Both yield close to perfect accuracy for the samples defined on the manifold (scatter plot), yet they differ significantly elsewhere. The shaded regions show the predicted class (Red or Blue) whereas contours show the confidence (absolute value of logits). \label{fig:learntfun}}
\end{figure}

\noindent Here, $\gamma_L([0, 1])$ defines the data-manifold and corresponds to a flower-shaped curve with $L$ petals, or a unit circle when $L = 0$ (see e.g. Fig~\ref{fig:learntfun}). Given a signal $\lambda : [0,1] \to \R$ defined on the latent space $[0, 1]$, the task entails learning a network $f : \R^2 \to \R$ such that $f \circ \gamma_L$ matches the signal $\lambda$.
\begin{figure*}[t]
\centering
\begin{subfigure}[t]{0.18\textwidth}
\includegraphics[width=1\linewidth]{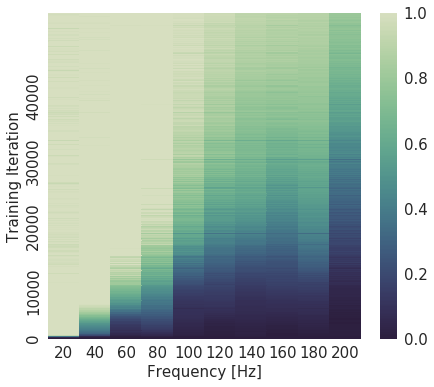}
\caption{\small $L = 0$}
\end{subfigure}
\begin{subfigure}[t]{0.18\textwidth}
\includegraphics[width=1\linewidth]{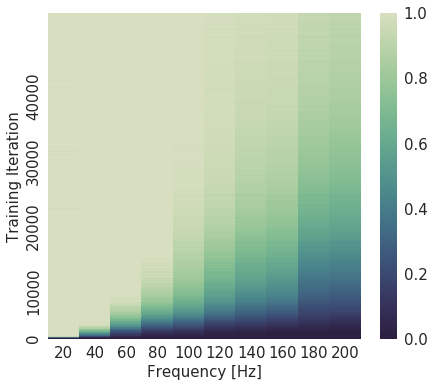}
\caption{\small $L = 4$}
\end{subfigure}
\begin{subfigure}[t]{0.18\textwidth}
\centering
\includegraphics[width=1\textwidth]{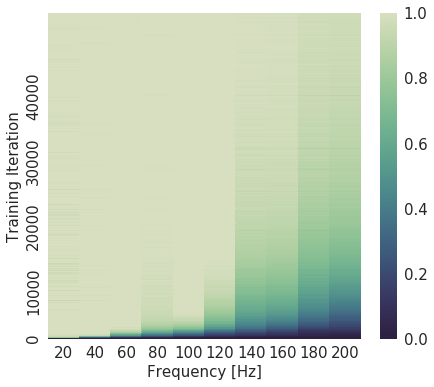}
\caption{\small $L = 10$}
\end{subfigure}
\begin{subfigure}[t]{0.18\textwidth}
\centering
\includegraphics[width=1\textwidth]{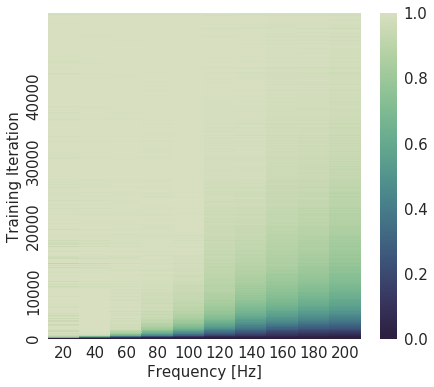}
\caption{\small $L = 16$}
\end{subfigure}
\begin{subfigure}[t]{0.18\textwidth}
\centering
  \includegraphics[width=1\textwidth]{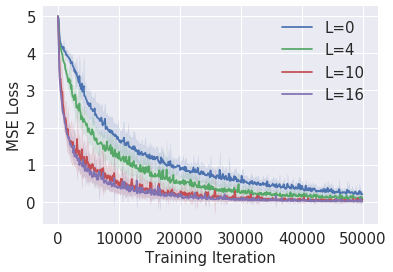}
  \caption{\small Loss curves} 
  \label{fig:lowfreqfirst_mfd_loss}
\end{subfigure}

\caption{\small (a,b,c,d): Evolution of the network spectrum (x-axis for frequency, colorbar for magnitude) during training (y-axis) for the same target functions defined on manifolds $\gamma_L$ for various $L$. Since the target function has amplitudes $A_i = 1$ for all frequencies $k_i$ plotted, the colorbar is clipped between 0 and 1. (e): Corresponding learning curves. 
\textbf{Gist}: Some manifolds (here with larger $L$) make it easier for the network to learn higher frequencies than others.} \label{fig:lowfrqfirst_mfd}
\end{figure*}

\begin{figure}[!htb]
\centering
  \includegraphics[width=0.65\linewidth]{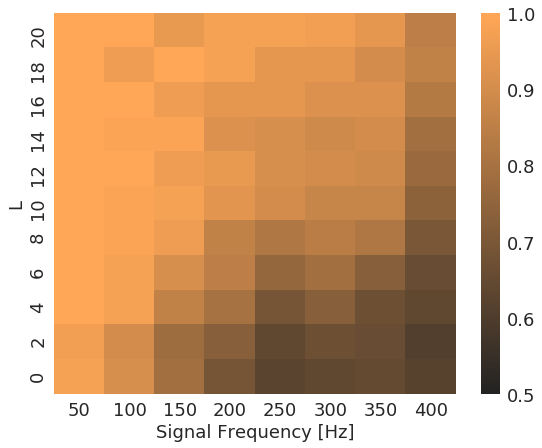}
  \caption{\small Heatmap of training accuracies of a network trained to predict the binarized value of a sine wave of given frequency (x-axis) defined on $\gamma_L$ for various $L$ (y-axis).\label{fig:lvk_acc_table}}
\end{figure}

\vspace{0.1cm}
\begin{experiment} \label{experiment:lowfrqfirst_mfd}
The set-up is similar to that of Experiment \ref{experiment:lowfreqfirst}, and $\lambda$ is as defined in Eqn.~\ref{target} with frequencies $\kappa = (20, 40, ..., 180, 200)$, and amplitudes $A_i = 1 \, \forall \, i$. The model $f$ is trained on the dataset $\{\gamma_L(z_i), \lambda(z_i)\}_{i = 1}^{N}$ with $N = 1000$ uniformly spaced samples $z_i$ between $0$ and $1$. The spectrum of $f \circ \gamma_L$ in expectation over $\varphi_i \sim U(0, 2\pi)$ is monitored as training progresses, and the result is shown in Fig~\ref{fig:lowfrqfirst_mfd} for various $L$. Fig~\ref{fig:lowfreqfirst_mfd_loss} shows the corresponding mean squared error curves. More experimental details in appendix \ref{app:experiment:lowfreqfirst_mfd}.

The results demonstrate a clear attenuation of the spectral bias as $L$ grows.  Moreover, Fig~\ref{fig:lowfreqfirst_mfd_loss}  suggests that the larger the $L$, the easier the learning task.
\end{experiment}

\begin{experiment} \label{experiment:lowfrqfirst_mfd_class}
Here, we adapt the setting of Experiment~\ref{experiment:lowfrqfirst_mfd} to binary classification by simply thresholding the function $\lambda$ at $0.5$ to obtain a binary target signal. To simplify visualization, we only use signals with a single frequency mode $k$, such that $\lambda(z) = \sin(2\pi k z + \varphi)$. We train the same network on the resulting classification task with cross-entropy loss\footnote{We use Pytorch's \texttt{BCEWithLogitsLoss}. Internally, it takes a sigmoid of the network's output (the logits) before evaluating the cross-entropy.} for $k \in \{50, 100, ..., 350, 400\}$ 
and $L \in \{0, 2, ..., 18, 20\}$. 
The heatmap in Fig~\ref{fig:lvk_acc_table} shows the classification accuracy for each $(k, L)$ pair. Fig~\ref{fig:learntfun} shows visualizations of  the functions learned by the same network, trained on $(k, L) = (200, 20)$ under identical conditions up to random initialization. 

Observe that increasing $L$ (i.e. going up a column in Fig~\ref{fig:lvk_acc_table}) results in better (classification) performance for the same target signal. This is the same behaviour as we observed in Experiment~\ref{experiment:lowfrqfirst_mfd} (Fig~\ref{fig:lowfrqfirst_mfd}a-d), but now with binary cross-entropy loss instead of the MSE.
\end{experiment}

\textbf{Discussion.} These experiments hint towards 
a rich interaction between the shape of the manifold and the effective difficulty of the learning task. 
The key mechanism underlying this phenomenon (as we formalize below) is that the relationship between frequency spectrum of the network $f$ and that of the fit $f \circ \gamma_L$ is mediated by the embedding map $\gamma_L$. In particular, we argue that a given signal defined on the manifold is easier to fit when the coordinate functions of the manifold embedding itself has high frequency components. Thus, in our experimental setting, the same signal embedded in a flower with more petals can be captured with lower frequencies of the network.

To understand this mathematically, we address the following questions: given a target function $\lambda$, how small can the frequencies of a solution $f$ be such that  $f \circ \gamma = \lambda$? And further, how does this relate to the geometry of the data-manifold $\cM$ induced by $\gamma$? 
To find out, we write the Fourier transform of the composite function,
\begin{align} \label{eqn:ftfuncomp}
\widetilde{(f \circ \gamma)}(\mathbf{l}) &= \int \mathbf{dk} \tilde{f}(\mathbf{k}) P_{\gamma}(\mathbf{l}, \mathbf{k}) \\ \mathrm{where} \; P_{\gamma}(\mathbf{l}, \mathbf{k}) &= \int_{[0, 1]^m} \mathbf{dz} \, e^{i(\mathbf{k} \cdot \gamma(\mathbf{z}) - \mathbf{l} \cdot \mathbf{z})} \nonumber
\end{align}
The kernel $P_{\gamma}$ depends on only $\gamma$ and elegantly encodes the correspondence between frequencies $ \mathbf{k} \in \R^d$ in input space and frequencies $\mathbf{l} \in \R^m$ in the latent space $[0, 1]^m$.  Following a procedure from \citet{bergnerspectral}, we can further investigate the behaviour of the kernel 
in the regime where the stationary phase approximation is applicable, i.e. when $l^2 + k^2 \to \infty$ (cf. section 3.2. of \citet{bergnerspectral}). In this regime, the integral $P_\gamma$ is dominated 
by critical points $\bar{\mathbf{z}}$ of its phase, which satisfy 
\beq \label{Pstationnary}
\bl = J_\gamma(\bar{\z}) \, \kk 
\eeq 
where $J_\gamma(\z)_{ij} = \nabla_i \gamma_j(\z)$ is the $m \times d$ Jacobian matrix of $\gamma$. Non-zero values of the kernel correspond to pairs $(\bl, \kk)$ such that Eqn~\ref{Pstationnary} has a solution. Further, given that the components of $\gamma$ (i.e. its coordinate functions) are defined on an interval $[0, 1]^m$, one can use their Fourier series representation together with Eqn~\ref{Pstationnary} to obtain a condition on their frequencies (shown in appendix \ref{app:ftfuncomp}). More precisely, we find that the $i$-th component of the RHS in Eqn~\ref{Pstationnary} is proportional to $\mathbf{p} \tilde{\gamma}_i[\mathbf{p}] k_i$ where $\mathbf{p} \in \mathbb{Z}^m$ is the frequency of the coordinate function $\gamma_i$. This yields that we can get arbitrarily large frequencies $l_i$ if $\tilde{\gamma}_i[\mathbf{p}]$ is large\footnote{Consider that the data-domain is bounded, implying that $\tilde{\gamma}$ cannot be arbitrarily scaled.} enough for large $\mathbf{p}$, even when $k_i$ is fixed. 

This is precisely what Experiments \ref{experiment:lowfrqfirst_mfd} and \ref{experiment:lowfrqfirst_mfd_class} demonstrate in a minimal setting. From Eqn~\ref{eqn:toy_mfd_dataset}, observe that the coordinate functions have a frequency mode at $L$. For increasing $L$, it is apparent that the frequency magnitudes $l$ (in the latent space) that can be expressed with the same frequency $k$ (in the input space) increases with increasing $L$. This allows the remarkable interpretation that the neural network function can express large frequencies on a manifold ($l$) with smaller frequencies w.r.t its input domain ($k$), provided that the coordinate functions of the data manifold embedding itself has high-frequency components. 

\section{Related Work}


A number of works have focused on showing that  neural networks are capable of approximating arbitrarily complex functions. \citet{hornik1989multilayer, cybenko1989approximation, leshno1993multilayer} have shown that neural networks can be universal approximators when given sufficient width; more recently, \citet{lu2017expressive} proved that this property holds also for width-bounded networks. \citet{montufar2014number} showed that the number of linear regions of deep ReLU networks 
grows polynomially with width and exponentially with depth;  \citet{raghu2016expressive} generalized this result and provided asymptotically tight bounds.  There have been various results of the benefits of depth for efficient approximation \citep{ganguli2016, Telgarsky2016, eldan2016power}.  
 These analysis on the expressive power of deep neural networks can in part explain why over-parameterized networks can perfectly learn random input-output mappings \citep{understanding_DL}. 

Our work more directly follows the line of research on implicit regularization in neural networks trained by gradient descent \citep{neyshabur2014search, soudry2017implicit, poggio2018theory, neyshabur2017exploring}.  
In fact, while our Fourier analysis of deep ReLU networks also reflects  the width and depth dependence of their expressivity, we focused on showing a learning bias of these networks towards simple functions with dominant lower frequency components. We view our results as a first step towards formalizing the findings of \citet{arpit2017closer}, where it is empirically shown that deep networks prioritize learning simple patterns of the data during training. 


A few other works studied neural networks through the lens of harmonic analysis. For example, \citet{candes1999harmonic} used the ridgelet transform to build constructive procedures for approximating a given function by neural networks, in the case of oscillatory activation functions. This approach has been recently generalized to unbounded activation functions by 
\citet{sonoda2017neural}. \citet{eldan2016power} use insights on the support of the Fourier spectrum of two-layer networks to derive a worse-case depth-separation result. \citet{barron1993universal} makes use of Fourier space properties of the target function to derive an architecture-dependent approximation bound. 
In a concurrent and independent work, \citet{Xu2018TrainingBO} make the same observation that lower frequencies are learned first. 
The subsequent work by \citet{Xu2018UnderstandingTA} proposes a theoretical analysis of the phenomenon in the case of 2-layer networks with sigmoid activation, based on the spectrum of the sigmoid function. 

In light of our findings, it is worth comparing the case of neural networks and other popular algorithms such that kernel machines (KM) and $K$-nearest neighbor classifiers. We refer to the 
Appendix \ref{app:kmknn} for a detailed discussion and references. In summary, our discussion there suggests that 1. DNNs strike a good balance between function smoothness and expressivity/parameter-efficiency compared with KM; 2. DNNs learn a smoother function compared with $K$NNs since the spectrum of the DNN decays faster compared with $K$NNs in the experiments shown there.

\section{Conclusion}

We studied deep ReLU networks through the lens of Fourier analysis. Several conclusions can be drawn from our analysis.  
While neural networks can approximate arbitrary functions, we find that they favour \emph{low frequency} ones -- hence they exhibit a bias towards smooth functions -- a phenomenon that we called \emph{spectral bias}. We also illustrated how the geometry of the data manifold impacts expressivity in a non-trivial way, as 
high frequency functions defined on  complex manifolds can be expressed by lower frequency network functions defined in input space. 

We view future work that explore the properties of neural networks in Fourier domain as promising. For example,  the Fourier transform affords a natural way of measuring how fast a function can change within a small neighborhood in its input domain; as such, it is a strong candidate for quantifying and analyzing the \emph{sensitivity} of a model -- which in turn provides a natural measure of complexity \citep{novak2018sensitivity}.  
We hope to encourage more research in this direction.

\clearpage
\newpage
\section*{Acknowledgements}
The authors would like to thank Joan Bruna, R\'emi Le Priol, Vikram Voleti, Ullrich K\"othe, Steffen Wolf, Lorenzo Cerrone, Sebastian Damrich, as well as  the anonymous reviewers for their valuable feedback. 

\bibliography{references}
\bibliographystyle{icml2019}

\clearpage
\newpage

\appendix    
\begin{appendix}

\section{Experimental Details}
\subsection{Experiment \ref{experiment:lowfreqfirst}} \label{app:experiment:lowfreqfirst}
We fit a 6 layer ReLU network with 256 units per layer $f_{\theta}$ to the target function $\lambda$, which is a superposition of sine waves with increasing frequencies: 
$$\lambda: [0, 1] \to \mathbb{R}, \, \lambda(z) = \sum_{i} A_i \sin(2 \pi k_i z + \varphi_i)$$ 
where $k_i = (5, 10, 15, ..., 50)$, and $\varphi_{i}$ is sampled from the uniform distribution $U(0, 2 \pi)$. In the first setting, we set equal amplitude for all frequencies, i.e. $A_i = 1 \, \forall \, i$, while in the second setting we assign larger amplitudes to the higher frequencies, i.e. $A_i = (0.1, 0.2, ..., 1)$. We sample $\lambda$ on 200 uniformly spaced points in $[0, 1]$ and train the network for $80000$ steps of full-batch gradient descent with Adam \citep{kingma2014adam}. Note that we do not use stochastic gradient descent to avoid the stochasticity in parameter updates as a confounding factor. We evaluate the network on the same 200 point grid every 100 training steps and compute the magnitude of its (single-sided) discrete fourier transform at frequencies $k_i$ which we denote with $|\tilde{f}_{k_i}|$. Finally, we plot in figure \ref{fig:lowfreqfirst} the normalized magnitudes $\frac{|\tilde{f}_{k_i}|}{A_i}$ averaged over 10 runs (with different sets of sampled phases $\varphi_i$). We also record the spectral norms of the weights at each layer as the training progresses, which we plot in figure \ref{fig:lowfreqfirst} for both settings (the spectral norm is evaluated with 10 power iterations). In figure \ref{fig:learntfuncs_eq_amps}, we show an example target function and the predictions of the network trained on it (over the iterations), and in figure \ref{fig:lowfrqfirst_loss_curves} we plot the loss curves.

\begin{figure}
\begin{subfigure}[t]{0.45\textwidth}
\centering
\includegraphics[width=0.99\textwidth]{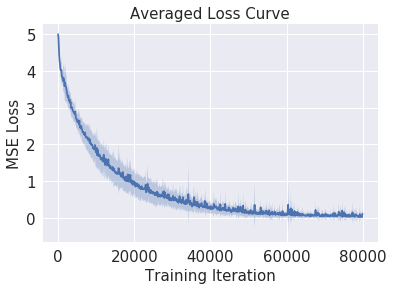}
\caption{\small Equal Amplitudes.}
\end{subfigure}
\hfill
\begin{subfigure}[t]{0.45\textwidth}
\centering
\includegraphics[width=0.99\textwidth]{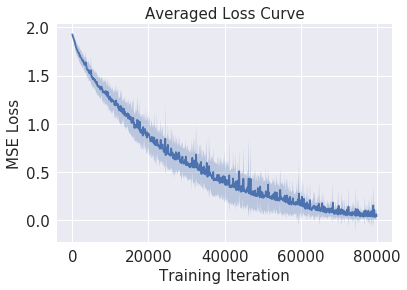}
\caption{\small Increasing Amplitudes.}
\end{subfigure}
\caption{\small Loss curves averaged over multiple runs. (cf. Experiment \ref{experiment:lowfreqfirst})} \label{fig:lowfrqfirst_loss_curves}
\end{figure}

\subsection{Experiment \ref{experiment:lowfrqfirst_mfd}} \label{app:experiment:lowfreqfirst_mfd}

We use the same 6-layer deep 256-unit wide network and define the target function 
$$\lambda: \mathcal{D} \to \mathbb{R}, \; z \mapsto \lambda(z) = \sum_{i} A_i \sin (2\pi k_i z + \varphi_i)$$
where $k_i = (20, 40, ..., 180, 200)$, $A_i = 1 \, \forall \, i$ and $\varphi \sim U(0, 2\pi)$. We sample $\phi$ on a grid with 1000 uniformly spaced points between 0 and 1 and map it to the input domain via $\gamma_L$ to obtain a dataset $\{(\gamma_L(z_j), \lambda(z_j))\}_{j = 0}^{999}$, on which we train the network with 50000 full-batch gradient descent steps of Adam. On the same 1000-point grid, we evaluate the magnitude of the (single-sided) discrete Fourier transform of $f_{\theta} \circ \gamma_L$ every 100 training steps at frequencies $k_i$ and average over 10 runs (each with a different set of sampled $z_i$'s). Fig~\ref{fig:lowfrqfirst_mfd} shows the evolution of the spectrum as training progresses for $L = 0, 4, 10, 16$, and Fig~\ref{fig:lowfreqfirst_mfd_loss} shows the corresponding loss curves. 

\subsection{Experiment \ref{experiment:lohifreqnoisemnist}} \label{app:lohifreqnoisemnist}
In Figure~\ref{fig:mnistnoise_train}, we show the training curves corresponding to Figure~\ref{fig:mnistnoise}. 
\begin{figure*}[h]
\centering
\begin{subfigure}[t]{0.245\textwidth}
\includegraphics[width=1\linewidth]{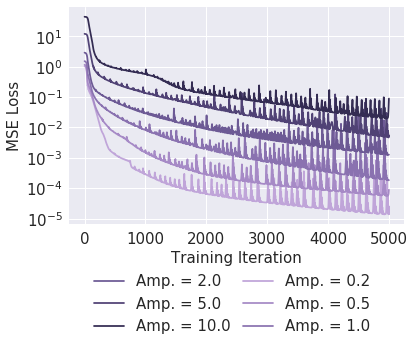}
\caption{\small $k = 0.1$ \label{fig:mnistnoise_train_a}}
\end{subfigure}
\begin{subfigure}[t]{0.245\textwidth}
\includegraphics[width=1\linewidth]{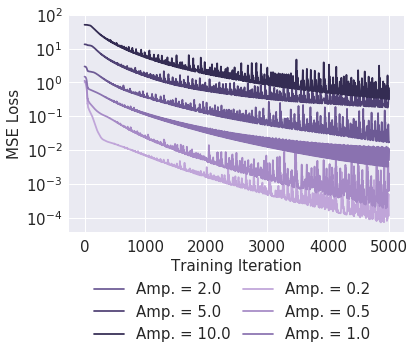}
\caption{\small $k = 1$ \label{fig:mnistnoise_train_b}}
\end{subfigure}
\begin{subfigure}[t]{0.245\textwidth}
\centering
\includegraphics[width=1\textwidth]{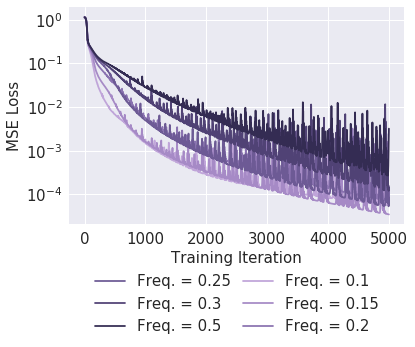}
\caption{\small $\beta = 0.5$ \label{fig:mnistnoise_train_c}}
\end{subfigure}
\begin{subfigure}[t]{0.245\textwidth}
\centering
\includegraphics[width=1\textwidth]{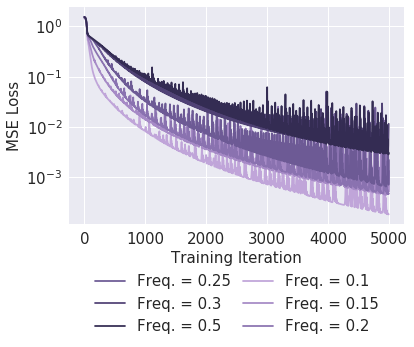}
\caption{\small $\beta = 1.$ \label{fig:mnistnoise_train_d}}
\end{subfigure}
\caption{\small (a,b,c,d): Training curves for various settings of noise amplitude $\beta$ and frequency $k$ corresponding to Figure~\ref{fig:mnistnoise}.} \label{fig:mnistnoise_train}
\end{figure*}

\subsection{Experiment \ref{experiment:kernel}} \label{app:kernel}
Consider the Gaussian Radial Basis Kernel, given by: 
\beq
k: X \times X \to \R,\, k_{\sigma}(\x, \y) \mapsto \exp \left(\frac{\|\x - \y\|}{\sigma^2}\right)
\eeq
where $X$ is a compact subset of $\R^d$ and $\sigma \in \R_{+}$ is defined as the width of the kernel\footnote{We drop the subscript $\sigma$ to simplify the notation.}. Since $k$ is positive definite \cite{fasshauer2011positive}, Mercer's Theorem can be invoked to express it as: 
\beq
k(\x, \y) = \sum_{n = 1}^{\infty} \lambda_i \varphi_n(\x)\varphi_n(\y)
\eeq
where $\varphi_n$ is the eigenfunction of $k$ satisfying: 
\beq
\int k(\x, \y) \varphi_n(\y) \mathbf{dy} = \langle k(\x, \cdot), \varphi_n \rangle = \lambda_n \varphi_n(\x)
\eeq
Due to positive definiteness of the kernel, the eigenvalues $\lambda_i$ are non-negative and the eigenfunctions $\varphi_n$ form an orthogonal basis of $L^2(X)$, i.e. $\langle \varphi_i, \varphi_j \rangle = \delta_{ij}$. The analogy to the final case is easily seen: let $X = {\x_i}_{i = 1}^{N}$ be the set of samples, $f: X \to \R$ a function. One obtains (cf. Chapter 4 \cite{rasmussen2004gaussian}): 
\beq
\langle k(\x, \cdot), f \rangle = \sum_{i = 1}^{N} k(\x, \x_i) f_i
\eeq
where $f_i = f(\x_i)$. Now, defining $K$ as the positive definite kernel matrix with elements $K_{ij} = k(\x_i \x_j)$, we consider it's eigendecomposition $V \Lambda V^T$ where $\Lambda$ is the diagonal matrix of (w.l.o.g sorted) eigenvalues $\lambda_1 \le ... \le \lambda_N$ and the columns of $V$ are the corresponding eigenvectors. This yields: 
\begin{align}
&k(\x_i, \x_j) = K_{ij} = (V \Lambda V^T)_{ij} = \sum_{n = 1}^{N} \lambda_n v_{ni} v_{nj} \nonumber \\
= &\sum_{n=1}^{N} \lambda_n \varphi_n(\x_i) \varphi_n(\x_j) \implies \varphi_{n}(\x_i) = v_{ni}
\end{align}
Like in \cite{braun2006model}, we define the \emph{spectrum} $\tilde{f}[n]$ of the function $f$ as: 
\beq
\tilde{f}[n] = \langle f, \varphi_n \rangle = \mathbf{f} \cdot \mathbf{v}_n
\eeq
where $\mathbf{f} = (f(\x_1), ..., f(\x_N))$. The value $n$ can be thought of a generalized notion of \emph{frequency}. Indeed, it is known \cite{fasshauer2011positive, rasmussen2004gaussian}, for instance, that the eigenfunctions $\varphi_n$ resemble sinusoids with increasing frequencies (for increasing $n$ or decreasing $\lambda_n$). In Figure~\ref{fig:freqofevecs}, we plot the eigenvectors $\mathbf{v}_0$ and $\mathbf{v}_N$ for $\{\x_i\}_{i=1}^{50}$ uniformly spaced between $[0, 1]$. Further, in Figure~? we evaluate the discrete Fourier transform of all $N = 50$ eigenvectors, and find that the eigenfunction index $n$ does indeed coincide with frequency $k$. Finally, we remark that the link between signal complexity and the spectrum is extensively studied in \cite{braun2006model}. 


\begin{figure}
\begin{subfigure}[t]{0.45\textwidth}
\centering
\includegraphics[width=0.99\textwidth]{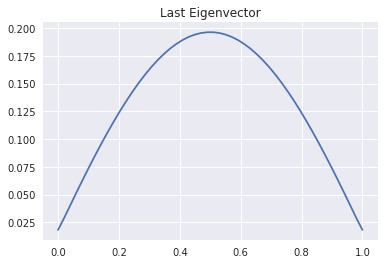}
\caption{\small Eigenvector with the largest eigenvalue ($n = 1$).}
\end{subfigure}
\hfill
\begin{subfigure}[t]{0.45\textwidth}
\centering
\includegraphics[width=0.99\textwidth]{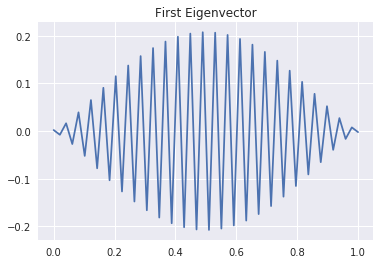}
\caption{\small Eigenvector with the smallest eigenvalue ($n = 50$).}
\end{subfigure}
\caption{\small Two extreme eigenvectors of the Gaussian RBF kernel for $50$ uniformly spaced samples between $0$ and $1$.} \label{fig:evplots}
\end{figure}

\subsubsection{Loss Curves Accompanying Figure~\ref{fig:kernelhfn}}

\begin{figure}[!h]
\centering
\includegraphics[width=0.35\textwidth]{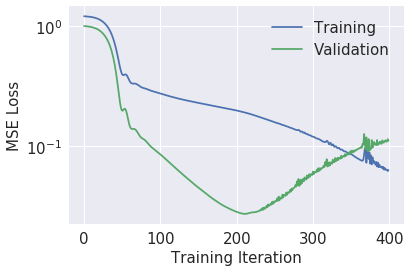}
\caption{\small Loss curves for the Figure~\ref{fig:kernelhfn}. We find that the validation loss dips at around the $200$th iteration. \label{fig:kernelhfn_loss}} 
\end{figure}

\subsection{Qualitative Ablation over Architectures} \label{app:arch_abl}
\begin{figure}
\begin{subfigure}[t]{0.45\textwidth}
\centering
\includegraphics[width=0.99\textwidth]{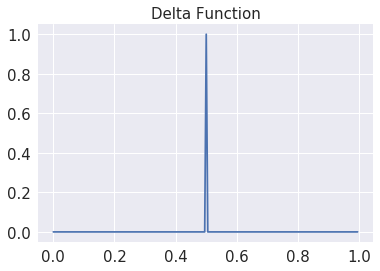}
\caption{\small Sampled $\delta$-function at $x = 0.5$. \label{fig:delta_peak:fn}}
\end{subfigure}
\hfill
\begin{subfigure}[t]{0.45\textwidth}
\centering
\includegraphics[width=0.99\textwidth]{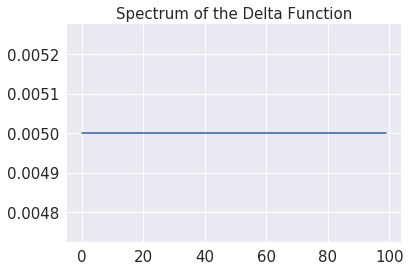}
\caption{\small Constant Spectrum of the $\delta$-function. \label{fig:delta_peak:spec}}
\end{subfigure}
\caption{\small The target function used in Experiment~\ref{experiment:arch_ablation}.} \label{fig:delta_peak}
\end{figure}
Theorem \ref{theorem:spectraldecay} exposes the relationship between the fourier spectrum of a network and its depth, width and max-norm of parameters. The following experiment is a qualitative ablation study over these variables. 

\begin{experiment} \label{experiment:arch_ablation}
In this experiment, we fit various networks to the $\delta$-function at $x = 0.5$ (see Fig~\ref{fig:delta_peak:fn}). Its spectrum is constant for all frequencies (Fig~\ref{fig:delta_peak:spec}), which makes it particularly useful for testing how well a given network can fit large frequencies. Fig~\ref{fig:spec_abl_clip} shows the ablation over weight clip (i.e. max parameter max-norm), Fig~\ref{fig:spec_abl_depth} over depth and Fig~\ref{fig:spec_abl_width} over width. Fig~\ref{fig:spec_abl_clip_pred} exemplarily shows how the network prediction evolves with training iterations. All networks are trained for 60K iterations of full-batch gradient descent under identical conditions (Adam optimizer with $lr = 0.0003$, no weight decay). 
\end{experiment}

\begin{figure*}
\begin{subfigure}[t]{0.24\textwidth}
\centering
\includegraphics[width=0.99\textwidth]{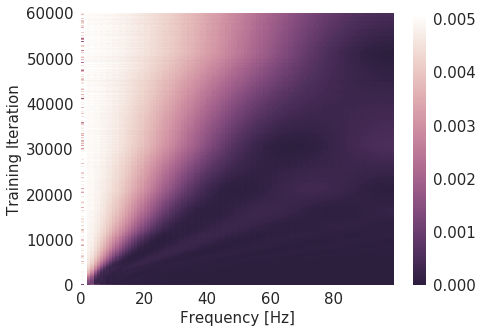}
\caption{\small Depth $ = 3$.}
\end{subfigure}
\hfill
\begin{subfigure}[t]{0.24\textwidth}
\centering
\includegraphics[width=0.99\textwidth]{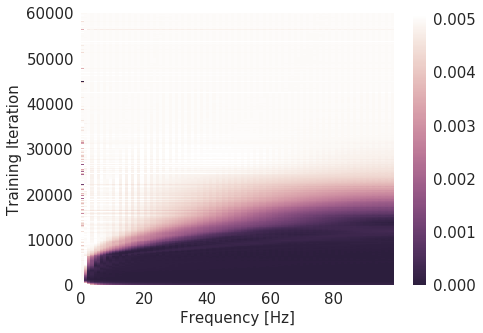}
\caption{\small Depth $ = 4$.}
\end{subfigure}
\hfill
\begin{subfigure}[t]{0.24\textwidth}
\centering
\includegraphics[width=0.99\textwidth]{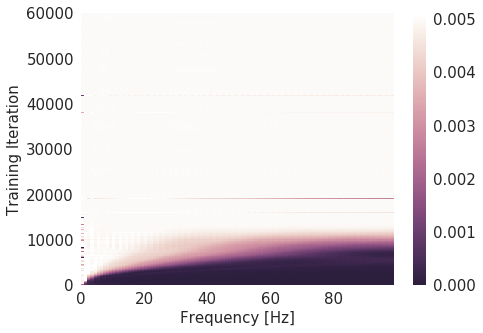}
\caption{\small Depth $ = 5$.}
\end{subfigure}
\hfill
\begin{subfigure}[t]{0.24\textwidth}
\centering
\includegraphics[width=0.99\textwidth]{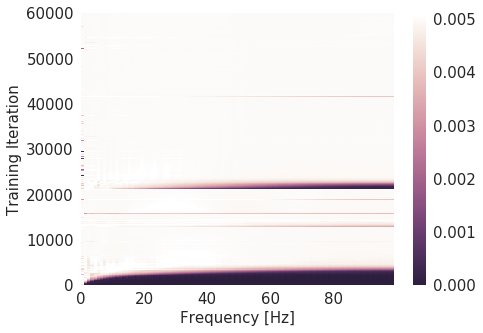}
\caption{\small Depth $ = 6$.}
\end{subfigure}
\caption{\small Evolution with training iterations (y-axis) of the Fourier spectrum (x-axis for frequency, and colormap for magnitude) for a network with \textbf{varying depth}, width $= 16$ and weight clip $= 10$. The spectrum of the target function is a constant $0.005$ for all frequencies.} \label{fig:spec_abl_depth}
\end{figure*}

\begin{figure*}
\begin{subfigure}[t]{0.24\textwidth}
\centering
\includegraphics[width=0.99\textwidth]{figures/spec_w=16-d=3-K=10.png}
\caption{\small Width $ = 16$.}
\end{subfigure}
\hfill
\begin{subfigure}[t]{0.24\textwidth}
\centering
\includegraphics[width=0.99\textwidth]{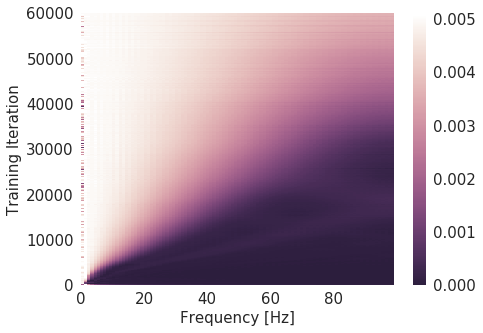}
\caption{\small Width $ = 32$.}
\end{subfigure}
\hfill
\begin{subfigure}[t]{0.24\textwidth}
\centering
\includegraphics[width=0.99\textwidth]{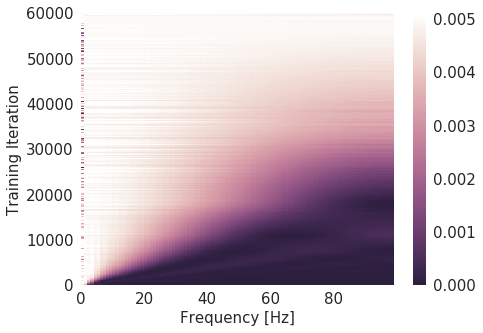}
\caption{\small Width $ = 64$.}
\end{subfigure}
\hfill
\begin{subfigure}[t]{0.24\textwidth}
\centering
\includegraphics[width=0.99\textwidth]{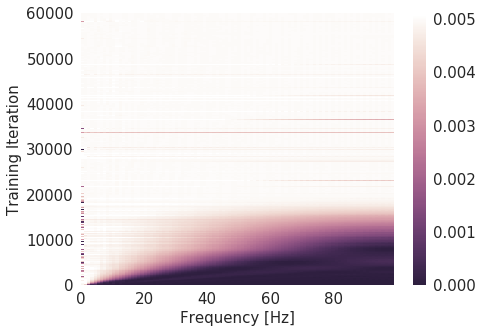}
\caption{\small Width $ = 128$.}
\end{subfigure}
\caption{\small Evolution with training iterations (y-axis) of the Fourier spectrum (x-axis for frequency, and colormap for magnitude) for a network with \textbf{varying width}, depth $= 3$ and weight clip $= 10$. The spectrum of the target function is a constant $0.005$ for all frequencies.} \label{fig:spec_abl_width}
\end{figure*}

\begin{figure*}
\begin{subfigure}[t]{0.24\textwidth}
\centering
\includegraphics[width=0.99\textwidth]{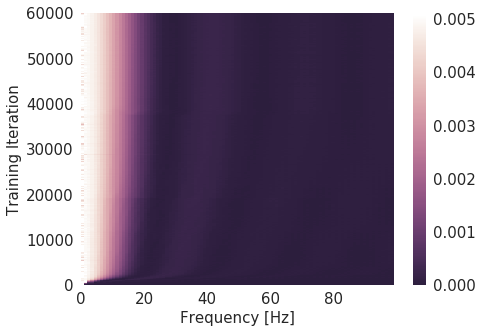}
\caption{\small Weight Clip $ = 0.1$.}
\end{subfigure}
\hfill
\begin{subfigure}[t]{0.24\textwidth}
\centering
\includegraphics[width=0.99\textwidth]{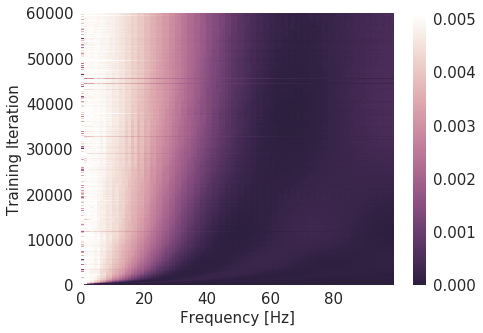}
\caption{\small Weight Clip $ = 0.15$.}
\end{subfigure}
\hfill
\begin{subfigure}[t]{0.24\textwidth}
\centering
\includegraphics[width=0.99\textwidth]{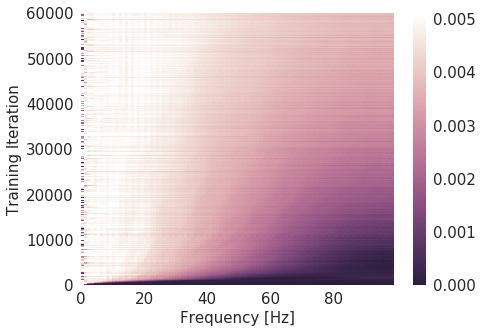}
\caption{\small Weight Clip $ = 0.2$.}
\end{subfigure}
\hfill
\begin{subfigure}[t]{0.24\textwidth}
\centering
\includegraphics[width=0.99\textwidth]{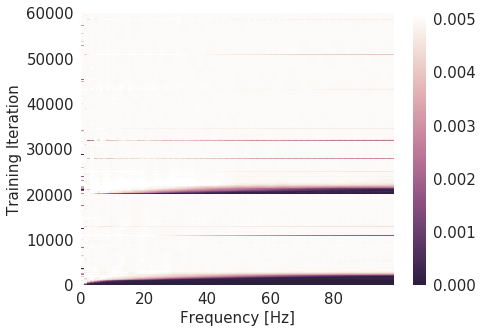}
\caption{\small Weight Clip $ = 2$.}
\end{subfigure}
\caption{\small Evolution with training iterations (y-axis) of the Fourier spectrum (x-axis for frequency, and colormap for magnitude) for a network with \textbf{varying weight clip}, depth $= 6$ and width $= 64$. The spectrum of the target function is a constant $0.005$ for all frequencies.} \label{fig:spec_abl_clip}
\end{figure*}

\begin{figure*}
\begin{subfigure}[t]{0.24\textwidth}
\centering
\includegraphics[width=0.99\textwidth]{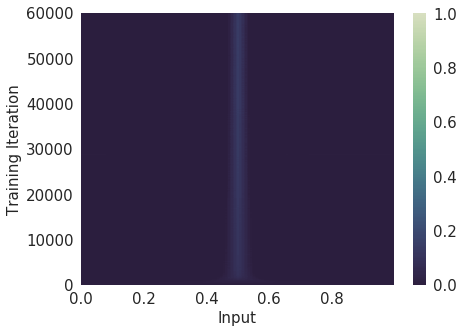}
\caption{\small Weight Clip $ = 0.1$.}
\end{subfigure}
\hfill
\begin{subfigure}[t]{0.24\textwidth}
\centering
\includegraphics[width=0.99\textwidth]{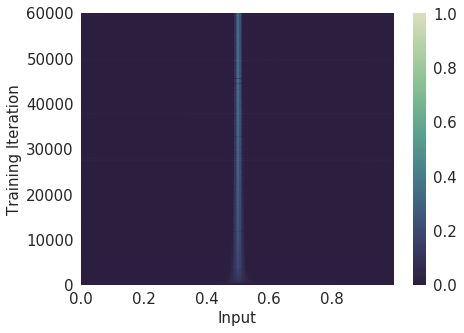}
\caption{\small Weight Clip $ = 0.15$.}
\end{subfigure}
\hfill
\begin{subfigure}[t]{0.24\textwidth}
\centering
\includegraphics[width=0.99\textwidth]{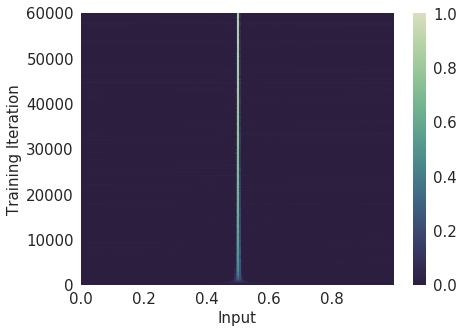}
\caption{\small Weight Clip $ = 0.2$.}
\end{subfigure}
\hfill
\begin{subfigure}[t]{0.24\textwidth}
\centering
\includegraphics[width=0.99\textwidth]{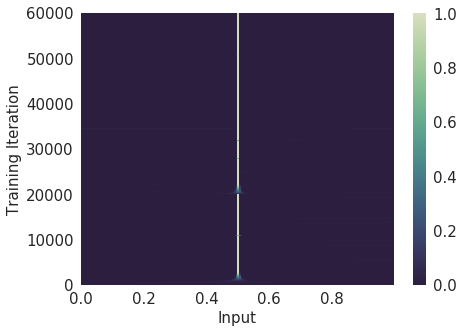}
\caption{\small Weight Clip $ = 2$.}
\end{subfigure}
\caption{\small Evolution with training iterations (y-axis) of the network prediction (x-axis for input, and colormap for predicted value) for a network with \textbf{varying weight clip}, depth $= 6$ and width $= 64$. The target function is a $\delta$ peak at $x = 0.5$.} \label{fig:spec_abl_clip_pred}
\end{figure*}

We make the following observations. 
\begin{enumerate}[label=(\alph*)]
    \item Fig~\ref{fig:spec_abl_depth} shows that increasing the depth (for fixed width) significantly improves the network's ability to fit higher frequencies (note that the depth increases linearly). 
    \item Fig~\ref{fig:spec_abl_width} shows that increasing the width (for fixed depth) also helps, but the effect is considerably weaker (note that the width increases exponentially). 
    \item Fig~\ref{fig:spec_abl_clip} shows that increasing the weight clip (or the max parameter max-norm) also helps the network fit higher frequencies. 
\end{enumerate}
The above observations are all consistent with Theorem~\ref{theorem:spectraldecay}, and further show that lower frequencies are learned first (i.e. the spectral bias, cf. Experiment~\ref{experiment:lowfreqfirst}). Further, Figure~\ref{fig:spec_abl_clip} shows that constraining the Lipschitz constant (weight clip) prevents the network from learning higher frequencies, furnishing evidence that the $\mathcal{O}(L_f)$ bound can be tight. 

\subsection{MNIST: A Proof of Concept} \label{app:mnist}
In the following experiment, we show that given two manifolds of the same dimension -- one flat and the other not -- the task of learning random labels is harder to solve if the input samples lie on the same manifold. We demonstrate on MNIST under the assumption that the manifold hypothesis is true, and use the fact that the spectrum of the target function we use (white noise) is constant in expectation, and therefore independent of the underlying coordinate system when defined on the manifold. 

\begin{experiment} \label{experiment:ae}
In this experiment, we investigate if it is easier to learn a signal on a more realistic data-manifold like that of MNIST (assuming the manifold hypothesis is true), and compare with a flat manifold of the same dimension. To that end, we use the $64$-dimensional feature-space $\mathcal{E}$ of a denoising\footnote{This experiment yields the same result if variational autoencoders are used instead.} autoencoder as a proxy for the real data-manifold of unknown number of dimensions. The decoder functions as an embedding of $\mathcal{E}$ in the input space $X = \mathbb{R}^{784}$, which effectively amounts to training a network on the reconstructions of the autoencoder. For comparision, we use an injective embedding\footnote{The xy-plane is $\mathbb{R}^3$ an injective embedding of a subset of $\mathbb{R}^2$ in $\mathbb{R}^3$.} of a 64-dimensional hyperplane in $X$. The latter is equivalent to sampling $784$-dimensional vectors from $U([0, 1])$ and setting all but the first 64 components to zero. The target function is white-noise, sampled as scalars from the uniform distribution $U([0, 1])$. Two identical networks are trained under identical conditions, and Fig~\ref{fig:ae_loss} shows the resulting loss curves, each averaged over 10 runs. 
\end{experiment}

This result complements the findings of \cite{arpit2017closer} and \cite{understanding_DL}, which show that it's easier to fit random labels to random inputs if the latter is defined on the full dimensional input space (i.e. the dimension of the flat manifold is the same as that of the input space, and not that of the underlying data-manifold being used for comparison).

\begin{figure*}[!h]
\centering
\includegraphics[width=0.5\textwidth]{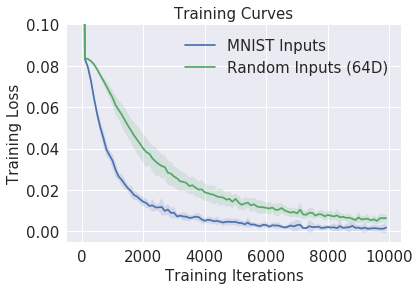}
\caption{\small Loss curves of two identical networks trained to regress white-noise under identical conditions, one on MNIST reconstructions from a DAE with 64 encoder features (blue), and the other on 64-dimensional random vectors (green). \label{fig:ae_loss}}
\end{figure*}

\subsection{Cifar-10: It's All Connected} \label{app:cifar10connected}

We have seen that deep neural networks are biased towards learning low frequency functions. This should have as a consequence that isolated \emph{bubbles} of constant prediction are rare. This in turn implies that given any two points in the input space and a network function that predicts the same class for the said points, there should be a path connecting them such that the network prediction does not change along the path. In the following, we present an experiment where we use a path finding method to find such a path between all Cifar-10 input samples indeed exist. 

\begin{figure*}
	\includegraphics[width=\textwidth]{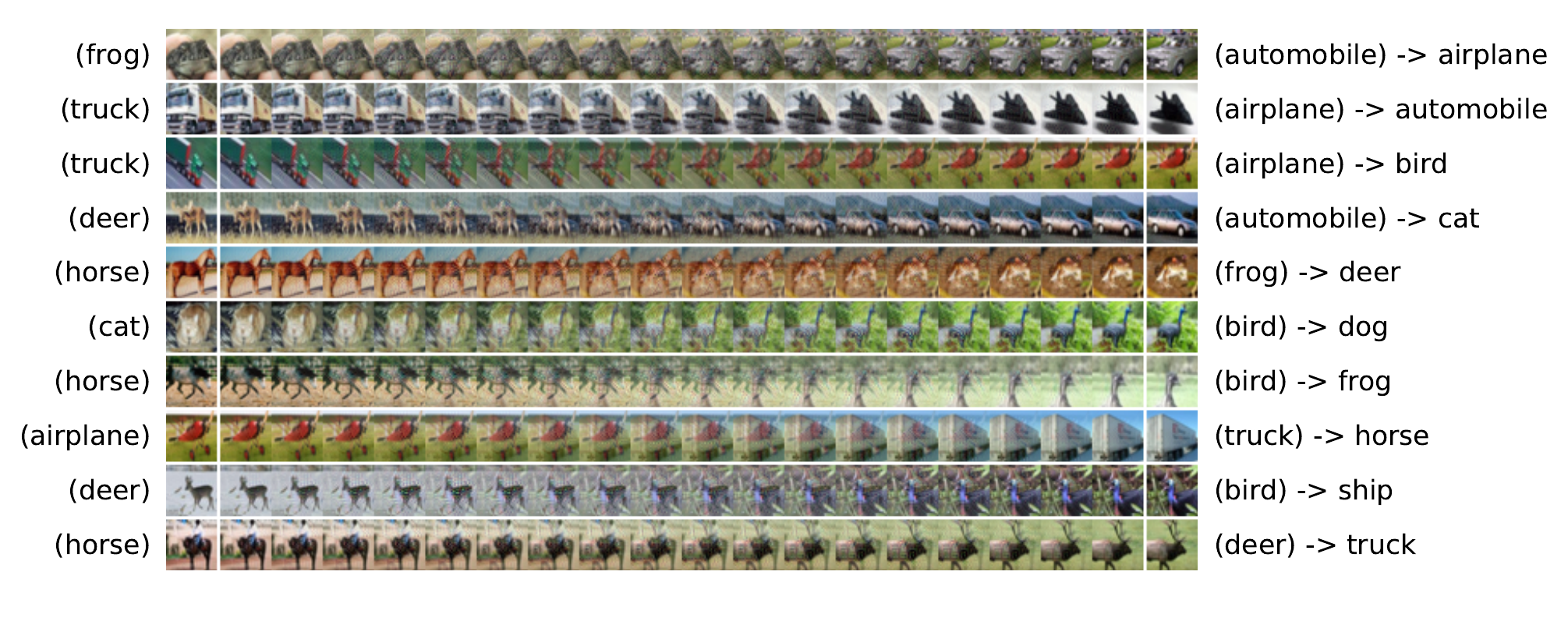}
    \caption{Path between CIFAR-10 adversarial examples (e.g. ``frog'' and ``automobile'', such that all images are classified as ``airplane'').}
    \label{fig:adv_adv_paths}
\end{figure*}

\begin{experiment}
Using AutoNEB \cite{kolsbjerg2016automated}, we construct paths between (adversarial) Cifar-10 images that are classified by a ResNet20 to be all of the same target class.
AutoNEB bends a linear path between points in some space $\mathbb{R}^m$ so that some maximum energy along the path is minimal.
Here, the space is the input space of the neural network, i.e. the space of $32\times32\times3$ images and the logit output of the ResNet20 for a given class is minimized.
We construct paths between the following points in image space:
\begin{itemize}
	\item From one training image to another,
    \item from a training image to an adversarial,
    \item from one adversarial to another.
\end{itemize}
We only consider pairs of images that belong to the same class $c$ (or, for adversarials, that originate from another class $\neq c$, but that the model classifies to be of the specified class $c$).
For each class, we randomly select 50 training images and select a total of 50 random images from all other classes and generate adversarial samples from the latter.
Then, paths between all pairs from the whole set of images are computed.

The AutoNEB parameters are chosen as follows:
We run four NEB iterations with 10 steps of SGD with learning rate $0.001$ and momentum $0.9$.
This computational budget is similar to that required to compute the adversarial samples.
The gradient for each NEB step is computed to maximize the logit output of the ResNet-20 for the specified target class $c$.
We use the formulation of NEB without springs \cite{draxler2018essentially}.

The result is very clear: We can find paths between \emph{all} pairs of images for all CIFAR10 labels that do not cross a single decision boundary.
This means that all paths belong to the same connected component regarding the output of the DNN. This holds for all possible combinations of images in the above list. Figure \ref{fig:adversarial_paths_all} shows connecting training to adversarial images and Figure \ref{fig:adv_adv_paths} paths between pairs of adversarial images. Paths between training images are not shown, they provide no further insight. Note that the paths are strikingly simple: Visually, they are hard to distinguish from the linear interpolation. Quantitatively, they are essentially (but not exactly) linear, with an average length $(3.0 \pm 0.3) \%$ longer than the linear connection.
\end{experiment}

\begin{figure*}
	\includegraphics[width=\textwidth]{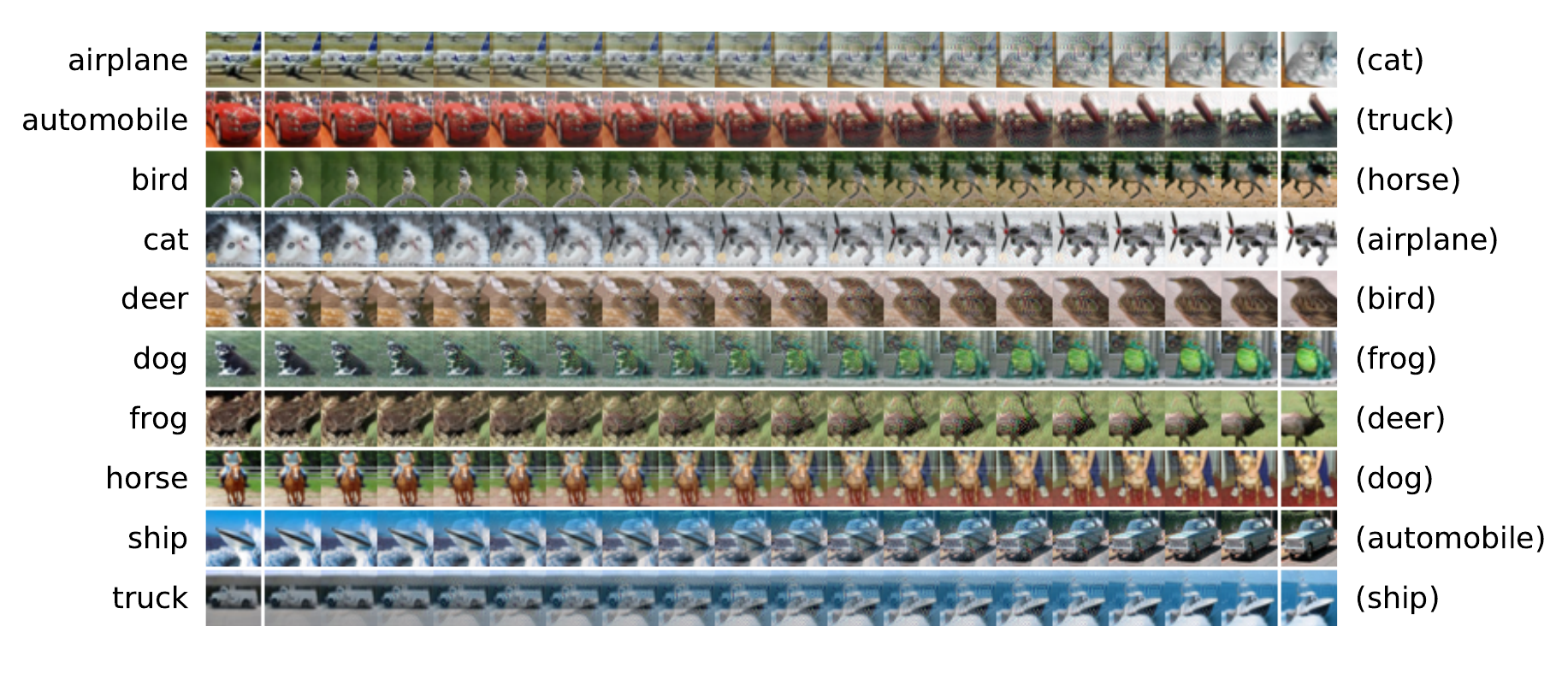}
    \caption{\small Each row is a path through the image space from an adversarial sample (right) to a true training image (left). All images are classified by a ResNet-20 to be of the class of the training sample on the right with at least 95\% softmax certainty. This experiment shows we can find a path from adversarial examples (right, Eg. "(cat)") that are classified as a particular class ("airplane") are connected to actual training samples from that class (left, "airplane") such that all samples along that path are also predicted by the network to be of the same class.}
    \label{fig:adversarial_paths_all}
\end{figure*}

\section{The Continuous Piecewise Linear Structure of Deep ReLU Networks} \label{app:cpwlrelu}

We consider the class of ReLU network functions $f: \R^d  \mapsto \R$ defined by Eqn.~\ref{reluDNN2}. 
Following the  terminology of \cite{raghu2016expressive, montufar2014number}, each linear region of the network then corresponds to a unique \emph{activation pattern}, wherein each hidden neuron is assigned an activation variable $\epsilon \in \{-1, 1\}$, conditioned on whether its input is positive or negative. 
ReLU networks can be explictly expressed as a sum over all possible activation patterns, as in the following lemma.

\begin{lemma} 
Given $L$ binary vectors $\sub{\epsilon}{1}, \cdots \sub{\epsilon}{L}$ with $\sub{\epsilon}{k} \in \{-1, 1\}^{d_k}$, let  ${T}^{(k)}_{\sub{\epsilon}{k}} : \R^{d_{k-1}} \to \R^{d_k}$ the affine function defined by $T^{(k)}_{\sub{\epsilon}{k}}(\uu)_i =  (T^{(k)}(\uu))_i$ if $(\epsilon_{k})_i = 1$, and $0$ otherwise. ReLU network functions, as defined in Eqn.~\ref{reluDNN2}, can be expressed as 
\beq \label{PWL}
f(\x) = \sum_{\sub{\epsilon}{1}, \cdots \sub{\epsilon}{L}}    1_{P_{f, \epsilon}}(\x) \, \left(T^{(L+1)} \circ T^{(L)}_{\sub{\epsilon}{L}} \circ \cdots \circ T^{(1)}_{\sub{\epsilon}{1}}\right)(\x)
\eeq
where $1_P$ denotes the indicator function of the subset $P\subset \R^d$,  and $P_{f, \epsilon}$ is the polytope defined as the set of solutions of the following linear inequalities (for all $k = 1, \cdots, L$):  
\beq
(\epsilon_{k})_i \,(T^{(k)}\circ T^{(k-1)}_{\sub{\epsilon}{{k-1}}}\circ \cdots \circ T^{(1)}_{\sub{\epsilon}{1}})(\x)_i \geq 0, \quad i=1, \cdots d_k
\eeq


\end{lemma}
$f$ is therefore affine on each of the polytopes $P_{f, \epsilon}$, which finitely partition the input space $\R^d$ to convex polytopes. Remarkably, the correspondence between ReLU networks and CPWL functions goes both ways: \citet{arora2018understanding} show that every CPWL function can be represented by a ReLU network, which in turn endows ReLU networks with the universal approximation property.

Finally, in the standard basis, each affine map  ${T}^{(k)}: \R^{d_{k-1}} \to \R^{d_k}$ is specified by a weight matrix $W^{(k)} \in \R^{d_{k-1}} \times \R^{d_{k}}$ and a bias vector $b^{(k)} \in \R^{d_k}$. 
In the linear region $P_{f, \epsilon}$,  $f$ can be expressed as
$f_\epsilon(x) = W_\epsilon x + b_\epsilon$, where in particular
\beq 
W_\epsilon = W^{(L+1)} W^{(L)}_{\epsilon_L} \cdots W^{(1)}_{\epsilon_1} \in \R^{1\times d}, 
\eeq 
where $W^{(k)}_{\epsilon}$ is obtained from $W^{(k)}$ by setting its $j$th column to zero whenever $(\epsilon_{k})_j = -1$. 

\section{Fourier Analysis of ReLU Networks} \label{app:ftrelu}

\subsection{Proof of Lemma \ref{FTRelu}}
\begin{proof} \textbf{Case 1: The function $f$ has compact support.} The vector-valued function $\kk f(\x) e^{i \kk \cdot \x}$ is continuous everywhere and  has well-defined and continuous gradients almost everywhere. So by Stokes' theorem (see e.g  \citet{Spivak71}), the integral of its divergence is a pure boundary term.  Since we restricted to functions with compact support, the theorem yields
\beq 
\int \nabla_\x \cdot \left[\kk f(\x) e^{-i \kk \cdot \x}  \right] \mathbf{dx} = 0 
\eeq 
The integrand is $(\kk \cdot (\nabla_\x f)(\x) - i k^2  f(\x)) e^{-i \kk \cdot \x}$, so we deduce, 
\beq
\hat{f}(\kk) = \frac{1}{-i k^2} \kk \cdot \! \int (\nabla_\x f)(\x)\, e^{-i \kk \cdot \x}
\eeq 
Now, within each polytope of the decomposition (\ref{PWL}), $f$ is affine so its gradient is a constant vector, 
$\nabla_\x f_\epsilon = W_\epsilon^T$, 
which gives the desired result (\ref{FTRelu}).

\textbf{Case 2: The function $f$ does not have compact support.} 
Without the assumption of compact support, the function $f$ is not squared-integrable. The Fourier transform therefore only exists in the sense of distributions, as defined below.  

Let $\mathcal{S}$ be the Schwartz space over $\R^d$ of rapidly decaying test functions which together with its derivatives decay to zero as $x \to \infty$ faster than any power of $x$. A tempered distribution is a continuous linear functional on $\mathcal{S}$. A function $f$ that doesn't grow faster than a polynomial at infinity can be identified with a tempered distribution $T_f$ as: 
\beq
T_f: \mathcal{S} \to \R, \, \varphi \mapsto \langle f, \varphi \rangle = \int_{\R^d} f(\x) \varphi(\x) \mathbf{dx}
\eeq
In the following, we shall identify $T_f$ with $f$. The Fourier transform $\tilde f$ of the tempered distribution is defined as: 
\beq
\langle \tilde f, \varphi \rangle := \langle f, \tilde \varphi \rangle
\eeq
where $\tilde \varphi$ is the Fourier transform of $\varphi$. In this sense, the standard notion of the Fourier transform is generalized to functions that are not squared-integrable. 

Consider the continuous piecewise-linear ReLU network $f: \R^d \to \R$. Since it can grow at most linearly, we interpret it as a tempered distribution on $\R^d$. 
Recall that the linear regions $P_{\epsilon}$ are enumerated by $\epsilon$. Let $f_{\epsilon}$ be the restriction of $f$ to $P_\epsilon$, making $f_{\epsilon}(\x) = W_{\epsilon}^T \x$. The distributional derivative of $f$ is given by: 
\beq \label{eq:cpwlgrad}
\nabla_\x f = \sum_{\epsilon} \nabla_\x f_\epsilon \cdot 1_{P_\epsilon} = \sum_{\epsilon} W_{\epsilon}^T 1_{P_\epsilon}
\eeq
where $1_{P_\epsilon}$ is the indicator over $P_{\epsilon}$ and we used $\nabla_\x f_{\epsilon} = W_{\epsilon}^T$.
It then follows from elementary properties of Schwartz spaces (see e.g. Chapter 16 of \citet{serov2017fourier}) that:
\begin{align}
[\widetilde{\nabla_{\x} f}](\kk) &= -i \kk \tilde{f}(\kk) \\
\implies \tilde{f}(\kk) &= \frac{1}{-ik^2}\kk \cdot [\widetilde{\nabla_{\x} f}](\kk)
\end{align}

Together with Eqn~\ref{eq:cpwlgrad} and linearity of the Fourier transform, this gives the desired result (\ref{FTRelu}). 

\end{proof}

\subsection{Fourier Transform of Polytopes} \label{app:ftpolytope}

\subsubsection{Theorem 1 of \citet{diaz2016fourier}}
Let $F$ be a $m$ dimensional polytope in $\mathbb{R}^d$, such that $1 \le m \le d$. Denote by $\mathbf{k} \in \mathbb{R}^d$ a vector in the Fourier space, by $\phi_{\mathbf{k}}(x) = -\mathbf{k} \cdot \mathbf{x}$ the linear phase function, by $\tilde{F}$ the Fourier transform of the indicator function on $F$, by $\partial F$ the boundary of $F$ and by $\text{vol}_m$ the $m$-dimensional (Hausdorff) measure. Let $\text{Proj}_F (\mathbf{k})$ be the orthogonal projection of $\mathbf{k}$ on to $F$ (obtained by removing all components of $\mathbf{k}$ orthogonal to $F$). Given a $m - 1$ dimensional facet $G$ of $F$, let $\mathbf{N}_F(G)$ be the unit normal vector to $G$ that points out of $F$. It then holds: 

1. If $\text{Proj}_F(\mathbf{k}) = 0$, then $\phi_{\mathbf{k}}(x) = \Phi_{\mathbf{k}}$ is constant on $F$, and we have:
\beq 
\tilde{F} = \text{vol}_F(F) e^{i \Phi_{\mathbf{k}}}
\eeq

2. But if $\text{Proj}_F(\mathbf{k}) \ne 0$, then: 
\beq 
\tilde{F} = i \sum_{G\in \partial F} \frac{\text{Proj}_F(\mathbf{k}) \cdot \mathbf{N}_{F}(G)}{\|\text{Proj}_F(\mathbf{k})\|^2} \tilde{G}(\mathbf{k})
\eeq

\subsubsection{Discussion}

The above theorem provides a recursive relation for computing the Fourier transform of an arbitrary polytope. More precisely, the Fourier transform of a $m$-dimensional polytope is expressed as a sum of fourier transforms over the $m-1$ dimensional boundaries of the said polytope (which are themselves polytopes) times a $\mathcal{O}(k^{-1})$ \emph{weight} term (with $k = \|\mathbf{k}\|$). The recursion terminates if $\text{Proj}_F(\mathbf{k}) = 0$, which then yields a constant. 

To structure this computation, \citet{diaz2016fourier} introduce a book-keeping device called the \emph{face poset} of the polytope. It can be understood as a weighted directed acyclic graph (DAG) with polytopes of various dimensions as its nodes. We start at the root node which is the full dimensional polytope $P$ (i.e. we initially set $m = n$). For all of the codimension-one boundary faces $F$ of $P$, we then draw an edge from the root $P$ to node $F$ and weight it with a term given by:
\beq 
W_{F, G} = i\frac{\text{Proj}_F(\mathbf{k}) \cdot \mathbf{N}_{F}(G)}{\|\text{Proj}_F(\mathbf{k})\|^2} 
\eeq
and repeat the process iteratively for each $F$. Note that the weight term is $\mathcal{O}(k^{-1})$ where $\text{Proj}_F(\mathbf{k}) \ne 0$. This process yields tree paths $T: F_0=P \to F_1 \to ... \to F_{|T|}$ where each $F_{i + 1} \in \partial F_i$ has one dimension less than $F_i$. For a given path and $\mathbf{k}$, the terminal node for this path, $F_{n_T}$, is the first polytope for which $\text{Proj}_{F_{n_T}}(\mathbf{k}) = 0$. 
The final Fourier transform is obtained by multiplying the weights along each path and summing over all tree paths:
\beq \label{eq:fullftpolytope}
\tilde{1}_P(\mathbf{k}) = \sum_{T} 
\prod_{i=0}^{|T|-1} 
W_{F_i, F_{i+1}}
\text{vol}_{F_{|T|}}(F_{|T|}) e^{i \Phi_{\mathbf{k}}}
\eeq 
where  $\Phi^{(T)} = \kk \cdot \x_0^T$ for an arbitrary point $\x_0^T$ in $F_{|T|}$.  

To write this as a weighted sum of indicator functions, as in Lemma \ref{lemma:ftpolyrat},  let $\mathcal T_n$ denote the set of all tree paths $T$ of length $n$, i.e. $|T| = n$.  For a tree path $T$, let  $S(T)$ be the orthogonal to the terminal node $F_n$, i.e the vectors $\kk$ such that $\text{Proj}_{F_n}(\mathbf{k}) = 0$. The sum over $T$ in Eqn (\ref{eq:fullftpolytope}) can be split as: 
\beq \label{eq:fullftpolytope2}
\tilde{1}_P = 
\sum_{n=0}^{d} \frac{1_{G_n}}{k^n} \sum_{T \in \mathcal T_n} 1_{S(T)} \prod_{i=0}^{n-1} \bar{W}_{F^{T}_i, F^{T}_{i + 1}} \text{vol}_{F^{T}_{n}}(F^{T}_{n}) e^{i \Phi^{(T)}_{\mathbf{k}}}
\eeq
where   $\bar{W}_{F, G} = k W_{F,G}$ and
$G_n = \bigcup_{T \in \mathcal T_n} S(T)$. 
In words, $G_n$ is the set of all vectors $\kk$ that are orthogonal to some $n$-codimensional face of the polytope.
We  identify:
\beq \label{eq:dndef}
D_q = \sum_{T \in \mathcal T_n} 1_{S(T)} \prod_{i=0}^{n-1} \bar{W}_{F^T_i, F^T_{i + 1}} \text{vol}_{F^T_n}(F^T_n) e^{i \Phi^{(T)}_{\mathbf{k}}}
\eeq
and $D_0(\kk) = \text{vol}(P)$ to obtain Lemma \ref{lemma:ftpolyrat}. Observe that $D_n$ depends on $k$ only via the phase term $e^{i \Phi^{(T)}_{\mathbf{k}}}$, implying that $D_n = \Theta(1) \, (k \to \infty)$. 

Informally, for a generic  vector $\mathbf{k}$, all paths terminate at the zero-dimensional vertices of the original polytope, i.e. $\text{dim}(F_n) = 0$, implying the length of the path $n$ equals the number of dimensions $d$, yielding a $\mathcal{O}(k^{-d})$ spectrum. The exceptions occur if a path terminates prematurely, because $\mathbf{k}$ happens to lie orthogonal to some $d - r$-dimensional face $F_r$ in the path, in which case we are left with a $\mathcal{O}(k^{-r})$ term (with $r < d$) which dominates asymptotically. Note that all vectors orthogonal to the $d-r$ dimensional face $F_r$ lie on a $r$-dimensional subspace of $\mathbb{R}^d$. Since a polytope has a finite number of faces (of any dimension), the $\mathbf{k}$'s for which the Fourier transform is $\mathcal{O}(k^{-r})$ (instead of $\mathcal{O}(k^{-d})$) lies on a finite union of closed subspaces of dimension $r$ (with $r < d$). The Lebesgue measure of all such lower dimensional subspaces  for all such $r$ is $0$, leading us to the conclusion that the spectrum decays as $\mathcal{O}(k^{-d})$ for \emph{almost all} $\mathbf{k}$'s. 

\subsection{On Theorem~\ref{theorem:spectraldecay}} \label{app:moarspectraldecay}
Equation \ref{FTbound} can be obtained by swapping the (finite) sum over $\epsilon$ in Lemma~\ref{FTRelu} with that over the paths $T$ in Eqn~\ref{eq:fullftpolytope2}. In particular, we have: 
\beq \label{eq:ftreludn}
\tilde f = \sum_{n=0}^{d} \frac{1_{H_n}}{k^{n + 1}} \sum_{\epsilon} W_{\epsilon} D_n^{\epsilon} 1_{G_n^\epsilon}
\eeq
Now, the sum $\sum_{\epsilon} W_{\epsilon} D_n^{\epsilon}(\hat \kk) I_{G_n^\epsilon}(\kk)$ is supported on the union:
\beq
H_n = \bigcup_{\epsilon} G_n^\epsilon
\eeq
Identifying: 
\begin{align}
C_n(\cdot, \theta) = \sum_{\epsilon} W_{\epsilon} D_n^{\epsilon} 1_{G_n^\epsilon}
\end{align}
where $C_n(\cdot, \theta) = \mathcal{O}(1)\,(k \to \infty)$, we obtain Theorem~\ref{theorem:spectraldecay}. Further, if $N_f$ is the number of linear regions of the network and $L_f = \max_{\epsilon} \|W_{\epsilon}\|$, we see that $C_{n} = \mathcal{O}(L_f N_f)$. Indeed, in Appendix \ref{app:arch_abl}, we empirically find that relaxing the constraint on the weight clip (which can be identified with $L_f$) enabled the network to fit higher frequencies, implying that the $\mathcal{O}(L_f)$ bound can be tight.

\subsection{Spectral Decay Rate of the Parameter Gradient} \label{app:graddecay}
\begin{proposition} \label{prop:graddecay}
Let $\theta$ be a generic parameter of the network function $f$. The spectral decay rate of $\nicefrac{\partial \tilde f}{\partial \theta}$ is $\mathcal{O}(k \tilde{f})$. 
\end{proposition}
\begin{proof}
For a fixed $\hat \kk$, observe from Eqn~\ref{eq:ftreludn} and Eqn~\ref{eq:dndef} that the only terms dependent on $k$ are the pure powers $k^{-n - 1}$ and the phase terms $e^{i\Phi_{\kk}^{(T)}}$, where $\Phi_{\kk}^{(T)} = k \hat \kk \cdot \x_0^{q(T)}$. However, the term $\x_0^{q(T)}$ is in general a function of $\theta$, and consequently the partial derivative of $e^{i\Phi_{\kk}^{(T)}}$ w.r.t $\theta$ yields a term that is proportional to $k$. This term now dominates the asymptotic behaviour as $k \to \infty$, adding an extra power of $k$ to the total spectral decay rate of $\tilde f$.
\end{proof}
Therefore, if $f = \mathcal{O}(k^{-\Delta - 1})$ where $\Delta$ is the codimension of the highest dimensional polytope $\hat \kk$ is orthogonal to, we have that $\nicefrac{\partial f}{\partial \theta} = \mathcal{O}(k^{-\Delta})$. 

\subsection{Convergence Rate of a Network Trained on Pure-Frequency Targets}
In this section, we derive an asymptotic bound on the convergence rate under the assumption that the target function has only one frequency component. 
\begin{proposition}
Let $\lambda: [0, 1] \to \R$ be a target function sampled in its domain at $N$ uniformly spaced points. Suppose that its Fourier transform after sampling takes the form: $\tilde \lambda(k) = A_0 \delta_{k, k_0}$, where $\delta$ is the Kronecker delta. Let $f$ be a neural network trained with full-batch gradient descent with learning rate $\eta$ on the Mean Squared Error, and denote by $f_t$ the state of the network at time $t$. Let $h(\cdot, t) = f_t - \lambda$ be the residual at time $t$. We have that: 
\beq
\left|\frac{\partial \tilde h(k_0, t)}{\partial t}\right| = \mathcal{O}(k_0^{-1})
\eeq
\end{proposition}

\begin{proof}
Consider that: 
\begin{align}
\left| \frac{\partial \tilde h(k_0)}{\partial t} \right| &= \left| \frac{\partial \tilde f(k_0)}{\partial \theta} \right| \left| \frac{\partial \theta}{\partial t}\right| \\
&= \left|\eta \frac{\partial \tilde f}{\partial \theta}\right| \left|\frac{\partial \mathcal{L}[\tilde f, \tilde \lambda]}{\partial \theta}\right|
\end{align}
where $\mathcal{L}$ is the sampled MSE loss and the first term is $\mathcal{O}(k_0^{-1})$ as can be seen from Proposition~\ref{prop:graddecay}. With Parceval's Theorem, we obtain: 
\begin{align}
\mathcal{L}[f, \lambda] &= \sum_{x = 0}^{N - 1} | f(x) - \lambda(x) |^2 = \sum_{k = -\nicefrac{N}{2}}^{\nicefrac{N}{2} - 1} | \tilde f(k) - \tilde \lambda(k) |^2 \nonumber \\
&= \mathcal{L}[\tilde f, \tilde \lambda]
\end{align}
For the magnitude of parameter gradient, we obtain: 
\begin{align}
\left|\frac{\partial \mathcal{L}[\tilde f, \tilde \lambda]}{\partial \theta}\right| &= 2 \left| \sum_{k = -\nicefrac{N}{2}}^{\nicefrac{N}{2} - 1} \text{Re}[\tilde f(k) - \tilde \lambda(k)] \frac{\partial \tilde f(k)}{\partial \theta} \right| \nonumber \\
&\le 2 \sum_{k = -\nicefrac{N}{2}}^{\nicefrac{N}{2} - 1} |\tilde f(k) - \tilde \lambda(k)| \left| \frac{\partial \tilde f(k)}{\partial \theta} \right| \nonumber \\
&\le 2\left| A_0 \frac{\partial \tilde f(k_0)}{\partial \theta} \right| + 2 \sum_{k = -\nicefrac{N}{2}}^{\nicefrac{N}{2} - 1} \left|\tilde f(k) \frac{\partial \tilde f(k)}{\partial \theta} \right|
\end{align}
where in the last line we used that $\tilde \lambda$ is a Kronecker-$\delta$ in the Fourier domain. Now, the second summand does not depend on $k_0$, but the first summand is again $\mathcal{O}(k_0^{-1})$.  
\end{proof}

\subsection{Proof of the Lipschtiz bound}
\label{App:Lipschitzbound}

\begin{proposition} \label{lemma:lipschitz}
The Lipschitz constant $L_f$ of the ReLU network $f$ is bound as follows (for all $\epsilon$): 
\beq
\|W_{\epsilon}\| \le L_f \le \prod_{k = 1}^{L+1} \|W^{(k)}\| \le \|\theta\|_{\infty}^{L+1} \sqrt{d} \prod_{k=1}^{L} d_k 
\eeq
\end{proposition}

\begin{proof} The first equality is simply the fact that $L_f = \max_{\epsilon} \|W_\epsilon\|$, and the second inequality follows trivially from the parameterization of a ReLU network as a chain of function compositions\footnote{Recall that the Lipschitz constant of a composition of two or more functions is the product of their respective Lipschtiz constants.}, together with the fact that the Lipschitz constant of the ReLU function is 1 (cf. \cite{miyato2018spectral}, equation 7). To see the third inequality, consider the definition of the spectral norm of a $I \times J$ matrix $W$: 
\beq
\|W\| = \max_{\|\mathbf{h}\| = 1} \|W\mathbf{h}\|
\eeq
Now, $\|W\mathbf{h}\| = \sqrt{\sum_i |\mathbf{w}_i \cdot \mathbf{h}|}$, where $\mathbf{w}_i$ is the $i$-th row of the weight matrix $W$ and $i = 1, ..., I$. Further, if $\|\mathbf{h}\| = 1$, we have $|\mathbf{w}_i \cdot \mathbf{h}| \le \|\mathbf{w}_i\| \|\mathbf{h}\| = \|\mathbf{w}_i\|$. Since $\|\mathbf{w}_i\| = \sqrt{\sum_j |w_{ij}|}$ (with $j = 1, ..., J$) and $|w_{ij}| \le \|\theta\|_{\infty}$, we find that $\|\mathbf{w}_i\| \le \sqrt{J} \|\theta\|_\infty$. Consequently, $\sqrt{\sum_i |\mathbf{w}_i \cdot \mathbf{h}|} \le \sqrt{IJ} \|\theta\|_{\infty}$ and we obtain: 

\beq
\|W\| \le \sqrt{IJ}\|\theta\|_{\infty}
\eeq

Now for $W = W^{(k)}$, we have $I = d_{k - 1}$ and $J = d_k$. In the product over $k$, every $d_k$ except the first and the last occur in pairs, which cancels the square root. For $k=1$, $d_{k - 1} = d$ (for the $d$ input neurons) and for $k = L+1$, $d_k = 1$ (for a single output neuron). The final inequality now follows. 
\end{proof} 

\subsection{The Fourier Transform of a Function Composition} \label{app:ftfuncomp}

Consider Equation \ref{eqn:ftfuncomp}. The general idea is to investigate the behaviour of $P_{\gamma}(\mathbf{l}, \mathbf{k})$ for large frequencies $\mathbf{l}$ on manifold but smaller frequencies $\mathbf{k}$ in the input domain. In particular, we are interested in the regime where the stationary phase approximation is applicable to $P_{\gamma}$, i.e. when $l^2 + k^2 \to \infty$ (cf. section 3.2. of \cite{bergnerspectral}). In this regime, the integrand in $P_\gamma(\mathbf{k}, \mathbf{l})$ oscillates fast enough such that the only constructive contribution originates from where the phase term $u(\mathbf{z}) = \mathbf{k} \cdot \gamma(\mathbf{z}) - \mathbf{l} \cdot \mathbf{z}$ does not change with changing $\mathbf{z}$. This yields the condition that $\nabla_{\mathbf{z}} u(\mathbf{z}) = 0$, which translates to the condition (with Einstein summation convention implied and $\partial_{\nu} = \nicefrac{\partial}{\partial x_{\nu}}$):
\beq \label{eqn:k_nu}
l_\nu = k_\mu \partial_\nu \gamma_\mu(\mathbf{z})
\eeq
Now, we impose periodic boundary conditions\footnote{This is possible whenever $\gamma$ is defined on a bounded domain, e.g. on $[0, 1]^m$.} on the components of $\gamma$, and without loss of generality we let the period be $2 \pi$. Further, we require that the manifold be contained in a box\footnote{This is equivalent to assuming that the data lies in a bounded set.} of some size in $\mathbb{R}^d$. The $\mu$-th component $\gamma_\mu$ can now be expressed as a Fourier series: 
\begin{align} \label{eqn:partial_gamma_mu}
\gamma_\mu(\mathbf{z}) &= \sum_{\mathbf{p} \in \mathbb{Z}^m} \tilde{\gamma}_\mu [\mathbf{p}] e^{-i p_\rho z_\rho} \nonumber\\
\partial_\nu \gamma_\mu(\mathbf{z}) &= \sum_{\mathbf{p} \in \mathbb{Z}^m} -i p_\nu \tilde{\gamma}_\mu [\mathbf{p}] e^{-i p_\rho z_\rho}
\end{align}
Equation \ref{eqn:partial_gamma_mu} can be substituted in equation \ref{eqn:k_nu} to obtain: 
\beq
l \hat{l}_\nu = -i k \sum_{\mathbf{p} \in \mathbb{Z}^m} p_\nu \hat{k}_\mu \tilde{\gamma}_\mu [\mathbf{p}] e^{-ip_\rho z_\rho}
\eeq
where we have split $k_\mu$ and $l_\nu$ in to their magnitudes $k$ and $l$ and directions $\hat{k}_\nu$ and $\hat{l}_\mu$ (respectively). We are now interested in the conditions on $\gamma$ under which the RHS can be large in magnitude, even when $k$ is fixed. Recall that $\gamma$ is constrained to a box -- consequently, we can not arbitrarily scale up $\tilde{\gamma}_{\mu}$. However, if $\tilde{\gamma}_{\mu}[\mathbf{p}]$ decays slowly enough with increasing $\mathbf{p}$, the RHS can be made arbitrarily large (for certain conditions on $\mathbf{z}$, $\hat{l}_{\mu}$ and $\hat{k}_{\nu}$).

\section{Volume of \emph{High-Frequency Parameters} in Parameter Space} \label{app:volparamspace}
For a given neural network, we now show that the volume of the parameter space containing parameters that contribute $\epsilon$-non-negligibly to frequency components of magnitude $k'$ above a certain cut-off $k$ decays with increasing $k$. For notational simplicity and without loss of generality, we absorb the direction $\hat{\mathbf{k}}$ of $\mathbf{k}$ in the respective mappings and only deal with the magnitude $k$. 

\begin{definition} \label{definition:frequencyparams}
Given a ReLU network $f_{\theta}$ of fixed depth, width and weight clip $K$ with parameter vector $\theta$, an $\epsilon > 0$ and $\Theta = B^{\infty}_{K}(0)$ a $L^{\infty}$ ball around $0$, we define:
$$
\Xi_{\epsilon}(k) = \{\theta \in \Theta | \exists k' > k, |\tilde{f}_{\theta}(k')| > \epsilon \}
$$
as the set of all parameters vectors $\theta \in \Xi_{\epsilon}(k)$ that contribute more than an $\epsilon$ in expressing one or more frequencies $k'$ above a \underline{cut-off frequency} $k$.
\end{definition}

\begin{remark} \label{remark:subsets}
If $k_2 \ge k_1$, we have $\Xi_{\epsilon}(k_2) \subseteq \Xi_{\epsilon}(k_1)$ and consequently $\text{vol}(\Xi_{\epsilon}(k_2)) \le \text{vol}(\Xi_{\epsilon}(k_1))$, where $\text{vol}$ is the Lebesgue measure. 
\end{remark}

\begin{lemma}
Let $1_{k}^{\epsilon}(\theta)$ be the indicator function on $\Xi_{\epsilon}(k)$. Then:
$$
\exists \, \kappa > 0:  \forall k \ge \kappa, 1_{k}^{\epsilon}(\theta) = 0
$$
\end{lemma}

\begin{proof}
From theorem \ref{theorem:spectraldecay}, we know that\footnote{Note from Theorem~\ref{theorem:spectraldecay} that $\Delta$ implicitly depends only on the unit vector $\hat \kk$.} $|\tilde{f}_{\theta}(k)| = \mathcal{O}(k^{-\Delta - 1})$ for an integer $1 \le \Delta \le d$. In the worse case where $\Delta = 1$, we have that $\exists M < \infty: |\tilde{f}_{\theta}(k)| < \frac{M}{k^2}$. Now, simply select a $\kappa > \sqrt{\frac{M}{\epsilon}}$ such that $\frac{M}{\kappa^2} < \epsilon$. This yields that $|\tilde{f}_{\theta}(\kappa)| < \frac{M}{\kappa^2} < \epsilon$, and given that $\frac{M}{\kappa^2} \le \frac{M}{k^2} \,\forall \, k \ge \kappa$, we find $|\tilde{f}_{\theta}(k)| < \epsilon \,\forall\, k \ge \kappa$. Now by definition \ref{definition:frequencyparams}, $\theta \not\in \Xi_{\epsilon}(\kappa)$, and since $\Xi_{\epsilon}(k) \subseteq \Xi_{\epsilon}(\kappa)$ (see remark \ref{remark:subsets}), we have $\theta \not\in \Xi_{\epsilon}(k)$, implying $1_{k}^{\epsilon}(\theta) = 0 \,\forall\, k \ge \kappa$.
\end{proof}

\begin{remark} \label{remark:indicatorbound}
We have $1_{k}^{\epsilon}(\theta) \le |\tilde{f}_{\theta}(k)|$ for large enough $k$ (i.e. for $k \ge \kappa$), since $|\tilde{f}_{\theta}(k)| \ge 0$.
\end{remark}
\setcounter{proposition}{0}
\begin{proposition}
The relative volume of $\Xi_{\epsilon}(k)$ w.r.t. $\Theta$ is $\mathcal{O}(k^{-\Delta - 1})$ where $1 \le \Delta \le d$. 
\end{proposition}
\begin{proof}
The volume is given by the integral over the indicator function, i.e. 
$$
\text{vol}(\Xi_{\epsilon}(k)) = \int_{\theta \in \Theta} 1_{k}^{\epsilon}(\theta) d\theta
$$

For a large enough $k$, we have from remark \ref{remark:indicatorbound}, the monotonicity of the Lebesgue integral and theorem \ref{theorem:spectraldecay} that: 

\begin{align*}
\text{vol}(\Xi_{\epsilon}(k)) &= \int_{\theta \in \Theta} 1_{k}^{\epsilon}(\theta) d\theta \\
&\le \int_{\theta \in \Theta} |\tilde{f}_{\theta}(k)| d\theta = \mathcal{O}(k^{-\Delta - 1}) \text{vol}(\Theta) \\
&\implies \frac{\text{vol}(\Xi_{\epsilon}(k))}{\text{vol}(\Theta)} = \mathcal{O}(k^{-\Delta - 1})     
\end{align*}
\end{proof}

\section{Kernel Machines and KNNs}
\label{app:kmknn}
In this section, in light of our findings, we want to compare DNNs with K-nearest neighbor (k-NN) classifier and kernel machines which are also popular learning algorithms, but are, in contrast to DNNs, better understood theoretically.

\subsection{Kernel Machines vs DNNs}
Given that we study why DNNs are biased towards learning smooth functions, we note that kernel machines (KM) are also highly Lipschitz smooth (Eg.~for Gaussian kernels all derivatives are bounded). However there are crutial differences between the two. While kernel machines can approximate any target function in principal \citep{hammer2003note}, the number of Gaussian kernels needed scales linearly with the number of sign changes in the target function \citep{bengio2009learning}. \citet{ma2017diving} have further shown that for smooth kernels, a target function cannot be approximated within $\epsilon$ precision in any polynomial of $1/\epsilon$ steps by gradient descent. 

Deep networks on the other hand are also capable of approximating any target function (as shown by the universal approximation theorems \citet{hornik1989multilayer, cybenko1989approximation}), but they are also parameter efficient in contrast to KM. For instance, we have seen that deep ReLU networks separate the input space into number of linear regions that grow polynomially in width of layers and exponentially in the depth of the network \citep{montufar2014number, raghu2016expressive}. A similar result on the exponentially growing expressive power of networks in terms of their depth is also shown in \citep{ganguli2016}. In this paper we have further shown that DNNs are inherently biased towards lower frequency (smooth) functions over a finite parameter space. This suggests that DNNs strike a good balance between function smoothness and expressibility/parameter-efficiency compared with KM.

\subsection{K-NN Classifier vs. DNN classifier}
$K$-nearest neighbor ($K$NN) also has a historical importance as a classification algorithm due to its simplicity. It has been shown to be a consistent approximator \cite{devroye1996consistency}, i.e., asymptotically its empirical risk goes to zero as $K \rightarrow \infty$ and $K/N \rightarrow 0$, where $N$ is the number of training samples. However, because it is a memory based algorithm, it is prohibitively slow for large datasets. Since the smoothness of a $K$NN prediction function is not well studied, we compare the smoothness between $K$NN and DNN. For various values of $K$, we train a $K$NN classifier on a $k  = 150$ frequency signal (which is binarized) defined on the $L = 20$ manifold (see section \ref{sec:notallmfds}), and extract probability predictions on a box interval in $\mathbb{R}^2$. On this interval, we evaluate the 2D FFT and integrate out the angular components (where the angle is parameterized by $\varphi$) to obtain $\zeta(k)$:  
\beq \label{radialfrq}
\zeta(k) = \frac{d}{dk} \int_{0}^{k} dk' k' \int_{0}^{2\pi} d\varphi |\tilde{f}(k', \varphi)|
\eeq
Finally, we plot $\zeta(k)$ for various $K$ in figure \ref{fig:knn}. Furthermore, we train a DNN on the very same dataset and overlay the radial spectrum of the resulting probability map on the same plot. We find that while DNN's are as expressive as a $K = 1$ KNN classifier at lower (radial) frequencies, the frequency spectrum of DNNs decay faster than KNN classifier for all values of $K$ considered, indicating that the DNN is smoother than the $K$NNs considered. We also repeat the experiment corresponding to Fig. \ref{fig:lvk_acc_table} with KNNs (see Fig. \ref{fig:lvk_acc_tables_knn}) for various $K$'s, to find that unlike DNNs, KNNs do not necessarily perform better for larger $L$'s, suggesting that KNNs do not exploit the geometry of the manifold like DNNs do.

\begin{figure*}
\begin{subfigure}[t]{0.45\textwidth}
\centering
\includegraphics[width=0.9\textwidth]{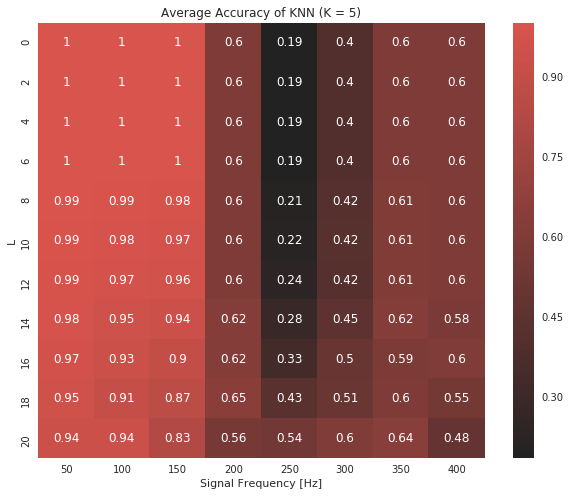}
\caption{\small $K = 5$.}
\end{subfigure}
\quad
\begin{subfigure}[t]{0.45\textwidth}
\centering
\includegraphics[width=0.9\textwidth]{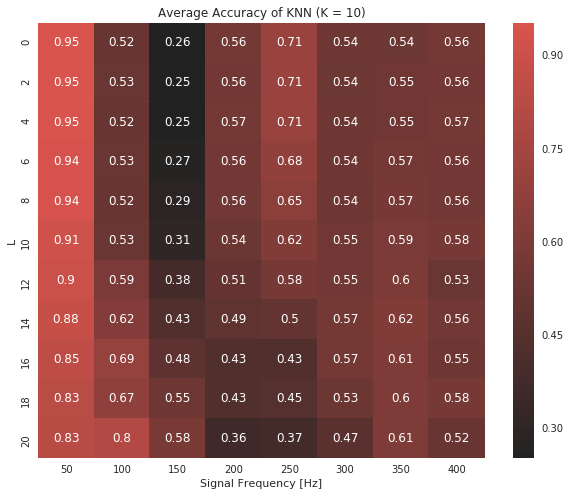}
\caption{\small $K = 10$.}
\end{subfigure}

\begin{subfigure}[t]{0.45\textwidth}
\centering
\includegraphics[width=0.9\textwidth]{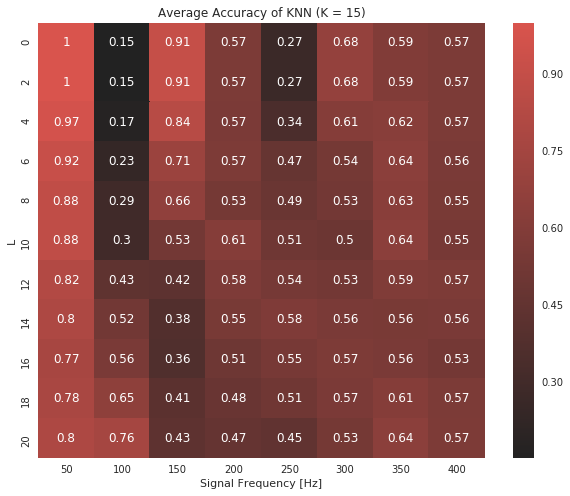}
\caption{\small $K = 15$.}
\end{subfigure}
\quad
\begin{subfigure}[t]{0.45\textwidth}
\centering
\includegraphics[width=0.9\textwidth]{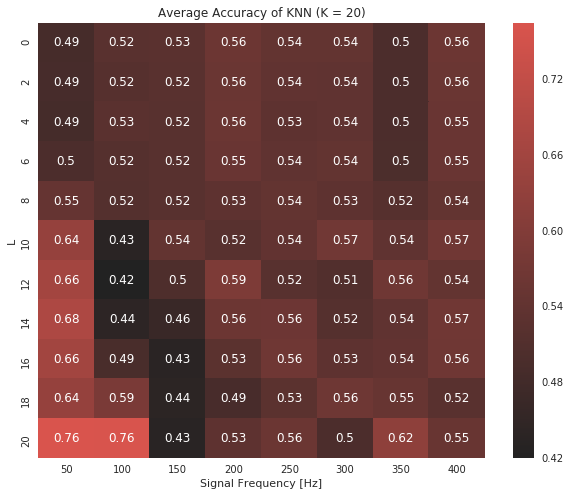}
\caption{\small $K = 20$.}
\end{subfigure}
\label{fig:lvk_acc_tables_knn}

\vspace{0.2cm}

\begin{subfigure}[t]{0.9\textwidth}
\centering
\includegraphics[width=0.7\textwidth,trim=0.in 0.in 0.in 0.in,clip]{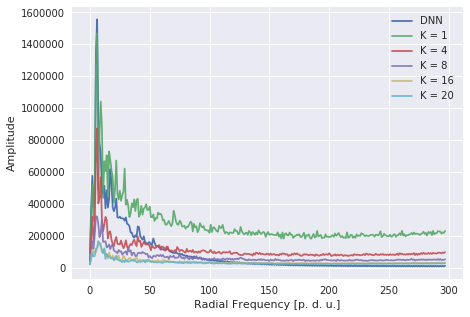}
\caption{\small Frequency spectrum \label{fig:knn}}
\end{subfigure}
\caption{(a,b,c,d): Heatmaps of training accuracies ($L$-vs-$k$) of KNNs for various $K$. When comparing with figure \ref{fig:lvk_acc_table}, note that the y-axis is flipped. (e): The frequency spectrum of $K$NNs with different values of $K$, and a DNN. The DNN learns a smoother function compared with the $K$NNs considered since the spectrum of the DNN decays faster compared with $K$NNs. \label{fig:lvk_acc_tables_knn}}
\end{figure*}

\end{appendix}

\end{document}